\documentclass{article}
\usepackage{iclr2027_conference,times}

\usepackage{amsmath,amsfonts,bm}

\def\eqref#1{equation~\ref{#1}}

\def\1{\bm{1}}

\DeclareMathAlphabet{\mathsfit}{\encodingdefault}{\sfdefault}{m}{sl}
\SetMathAlphabet{\mathsfit}{bold}{\encodingdefault}{\sfdefault}{bx}{n}

\usepackage{hyperref}
\usepackage{url}
\usepackage{booktabs}
\usepackage{multirow}
\usepackage{graphicx}
\usepackage{svg}
\usepackage{float}
\usepackage{amssymb}
\usepackage{amsthm}
\usepackage{enumitem}
\usepackage{amsthm}
\newtheorem{theorem}{Theorem}

\title{SINO: Scale-Invariant Neural Operator}

\author{
	Kaichen Ouyang$^{1,2}$ \\
	$^{1}$Westlake University \\
	$^{2}$University of Science and Technology of China \\
	\texttt{ouyangkaichen@westlake.edu.cn} \\
	\texttt{oykc@mail.ustc.edu.cn}
	\And
	Chenglei Yu$^{1}$ \\
	$^{1}$Westlake University \\
	\texttt{yuchenglei@westlake.edu.cn}
	\AND
	Chuanrui Wang$^{1}$ \\
	$^{1}$Westlake University \\
	\texttt{wangchuanrui@westlake.edu.cn}
	\And
	Tailin Wu$^{1}$\thanks{Corresponding author.} \\
	$^{1}$Westlake University \\
	\texttt{wutailin@westlake.edu.cn}
}

\iclrfinalcopy 

\begin{document}

\maketitle

\begin{abstract}
	In scientific machine learning, physical fields governed by partial differential equations exhibit low-rank structure and scale invariance. When solving equations on coarse grids, missing information leads to the closure problem—modeling unresolved physics to recover lost dynamics. Although closure terms depend on grid resolution, they represent scale-invariant physical laws. A model truly learning physics should capture these mechanisms with low-rank parameterization rather than memorizing grid-specific patterns. Inspired by this, we propose the Scale-Invariant Neural Operator (SINO), which learns on normalized physical scales via a dual-branch architecture operating in spectral and spatial domains. SINO uses bottleneck MLPs to generate continuous convolution kernels, embedding explicit low-rank inductive bias that concentrates $>$95\% variance in 2-3 modes, as validated by PCA across benchmarks, while drastically reducing parameters. This principled design yields 38$\times$ steeper scaling law exponents than FNO, demonstrating superior parameter efficiency. We compare SINO with traditional models (U-Net, DeepONet), Transformer models (Transolver, Oformer, GK-Transformer), and frequency-domain models (FNO, AMFNO, UFNO) on closure problems spanning externally forced Burgers turbulence, decaying Burgers turbulence, KS turbulence, Kolmogorov-forced NS turbulence, and decaying NS turbulence. Experiments show SINO achieves 1.5--38$\times$ error reduction and 2--23$\times$ parameter efficiency over baselines, with superior scaling laws reflecting exceptional data efficiency from principled low-rank design. Code is available at \url{https://github.com/AI4Science-WestlakeU/SINO}.
\end{abstract}

\section{Introduction}
Scientific machine learning aims to understand and simulate complex systems governed by physical laws using data-driven methods~\citep{karniadakis2021physics}. Unlike natural images, physical fields such as velocity, pressure, and vorticity evolve under partial differential equations and conservation laws. Although discretized fields may possess extremely high dimensionality, their dynamics are organized by a small number of mechanisms: conservation, transport, dissipation, and cross-scale energy transfer~\citep{holmes2012turbulence,berkooz1993proper,lumley1967structure}. This reveals that physical laws possess an effective \textbf{low-rank structure}, where a few degrees of freedom govern the overall dynamics~\citep{thibeault2024low}.

This low-rank nature imposes two requirements on model design: (1) capturing governing mechanisms with compact parameterization; (2) learning \textbf{resolution-independent} physical laws, where changing grid resolution alters the discrete representation but not the underlying equations. The same continuous process, when observed at different resolutions, is still governed by the same mechanisms.

The \textbf{closure problem} provides an ideal testbed for these principles~\citep{sanderse2024scientific}. High-fidelity direct numerical simulation (DNS) resolves small-scale structures but is computationally expensive; practical simulations use coarser grids. The coarse-grid state $\bar{u}_{\Delta} = \mathcal{P}_{\Delta}u$ requires a \textbf{closure term} $\tau$~\citep{leonard1975energy} to compensate for unresolved scales. The key insight is: although $\tau$ depends on the coarse-graining scale $\Delta$, the underlying physical mechanisms are governed by \textbf{scale-invariant laws}~\citep{meneveau2000scale,kolmogorov1991local}. These processes retain the same mathematical form across resolutions, differing only in operational scales. Therefore, models that truly learn physics should capture the low-rank, scale-invariant physical laws rather than memorizing resolution-specific patterns. 

Learning-based closure methods fall into two categories: \textbf{learned correction}~\citep{um2020solver,list2022learned} directly modifies coarse-grid states or evolution results, while \textbf{learned interpolation}~\citep{bar2019learning,kochkov2021machine} learns missing discretization schemes. Recent methods like the Indirect Neural Corrector (INC)~\citep{wei2026inc}, a learned correction approach, incorporate neural predictions as right-hand side terms in governing equations, improving long-term stability. However, existing methods focus on \emph{how} to couple predictions into solvers, not \emph{what} neural architecture should learn closure. The key challenge is: \textbf{how to represent closure operators that are simultaneously low-rank, scale-invariant, and capable of balancing global spectral structure with local spatial features?} Standard approaches fail this requirement: CNNs~\citep{krizhevsky2012imagenet} use resolution-dependent discrete kernels with $O(K^2 D^2)$ parameters per layer; FNO~\citep{li2020fourier} employs mode-specific Fourier weights with $O(k_{\max} D^2)$ parameters that scale with truncation; Transformers rely on high-dimensional attention without explicit low-rank constraints.

To address this challenge, we propose the \textbf{Scale-Invariant Neural Operator (SINO)} for learning closure terms that capture cross-resolution, scale-invariant physical laws. SINO employs a dual-branch architecture: a frequency branch captures global spectral energy transfer, while a spatial branch handles localized nonlinear structures. Both branches generate continuous convolution kernels via hypernetworks with bottleneck layers—taking normalized physical coordinates as input rather than discrete grid indices. This design enforces three principles: (1) \textbf{Scale invariance}: normalized coordinates ensure identical representations across resolutions; (2) \textbf{Low-rank inductive bias}: bottleneck layers with dimensions $n_f, n_s \ll D^2$ force kernels into low-dimensional subspaces, reducing parameters from $O(k_{\max} D^2)$ (FNO) or $O(K^2 D^2)$ (CNN) to $O((n_f + n_s) D^2)$ independent of resolution; (3) \textbf{Global-local balance}: dual branches synergistically combine frequency-domain long-range dependencies with spatial-domain localized patterns.

We validate SINO on six turbulent closure benchmarks spanning forced and decaying Burgers, Kuramoto-Sivashinsky, and Navier-Stokes equations at Reynolds numbers 1000-4000, comparing against eight baselines. SINO achieves 1.5-38$\times$ error reduction with 2-23$\times$ fewer parameters than FNO-based methods. Theoretical analysis proves SINO's bottleneck architecture enforces hard rank constraints (Theorem~\ref{thm:explicit_rank_constraint}) and Lipschitz continuity (Theorem~\ref{thm:lipschitz_hypernetworks}), with operator norms growing sublinearly ($O(\sqrt{n_f})$) compared to FNO's linear growth. Empirical PCA~\citep{abdi2010principal} reveals learned kernels concentrate $>$95\% variance in 2-3 principal components, validating extreme compression. Scaling law analysis shows SINO's power-law exponents are 3-38$\times$ steeper than baselines, demonstrating superior parameter efficiency. Our contributions are:

\begin{itemize}[leftmargin=*, itemsep=2pt]
\item \textbf{Architecture}: We propose SINO with bottleneck hypernetworks generating continuous kernels from normalized coordinates, embedding low-rank and scale-invariant inductive biases. Dual branches balance global frequency-domain and local spatial-domain representations.

\item \textbf{Empirical validation}: SINO achieves best average performance on six turbulent closure problems (forced/decaying Burgers, KS, forced/decaying NS), with 1.5-38$\times$ error reduction and 2-23$\times$ parameter efficiency over eight baselines (U-Net~\citep{ronneberger2015u}, DeepONet~\citep{lu2021learning}, Transolver~\citep{wu2024transolver}, OFormer~\citep{li2022transformer}, GK-Transformer~\citep{cao2021choose}, FNO~\citep{li2020fourier}, AM-FNO~\citep{xiao2024amortized}, UFNO~\citep{wen2022u}). Ablation studies confirm necessity of dual-branch fusion.

\item \textbf{Theoretical foundations}: We prove bottleneck layers enforce rank-$n_f$ and rank-$n_s$ constraints (Theorem~\ref{thm:explicit_rank_constraint}), Lipschitz continuity (Theorem~\ref{thm:lipschitz_hypernetworks}), and sublinear operator norm growth (Theorems~\ref{thm:frequency_domain_operator_norm}--\ref{thm:spatial_domain_operator_norm}). PCA analysis shows $>$95\% variance in 2-3 components. Scaling laws ($\alpha=0.568$ on NS vs. 0.015 for FNO) demonstrate SINO prioritizes learning physical laws over memorizing patterns during capacity growth.
\end{itemize}

\section{Related Work}

\paragraph{Closure Modeling.}
Resolving the finest spatiotemporal scales in PDEs is often computationally intractable, necessitating coarse-grained approximations. Traditional approaches like Reynolds-averaged Navier-Stokes (RANS), large-eddy simulation (LES)~\citep{heinz2020review}, and subgrid-scale (SGS) models~\citep{shankar2023differentiable} model unresolved physics through closure terms, but deriving reliable closures remains challenging with limited accuracy for complex flows~\citep{wei2026inc}. Recent machine learning approaches systematically learn closure operators from data. The indirect neural corrector (INC)~\citep{wei2026inc} embeds learned terms into the PDE's right-hand side for improved stability. We extend this framework with low-rank, scale-invariant architectures. Given coarse-grid state $\bar{u}_\Delta$, the learned closure predicts the correction term $\tau_\Delta = \text{SINO}(\bar{u}_\Delta; \Theta)$ that augments the coarse-grid dynamics: $\partial \bar{u}_\Delta / \partial t = \mathcal{L}_\Delta(\bar{u}_\Delta) + \tau_\Delta$, enabling stable rollout integration.

\paragraph{Implicit Parameterization.}
Implicit parameterization represents neural network weights or convolutional kernels as continuous functions of coordinates, decoupling learnable parameters from kernel size or resolution. Hypernetworks generate weights via an auxiliary network, enabling adaptive filter generation~\citep{ha2016hypernetworks,ma2022hyper}. Continuous Kernel Convolution (CKConv)~\citep{romero2021ckconv} models kernels as continuous functions, handling arbitrarily long sequences and irregular data. Neural Implicit Frequency Filters (NIFF)~\citep{grabinski2024large} learn filters in the frequency domain, enabling infinitely large kernels. We extend this paradigm by employing bottleneck hypernetworks to generate continuous kernels in both spatial and spectral domains, enforcing low-rank constraints aligned with turbulent closure physics while achieving cross-resolution scale invariance through normalized coordinates.

\begin{figure}[t]
	\centering
	\includegraphics[width=\textwidth]{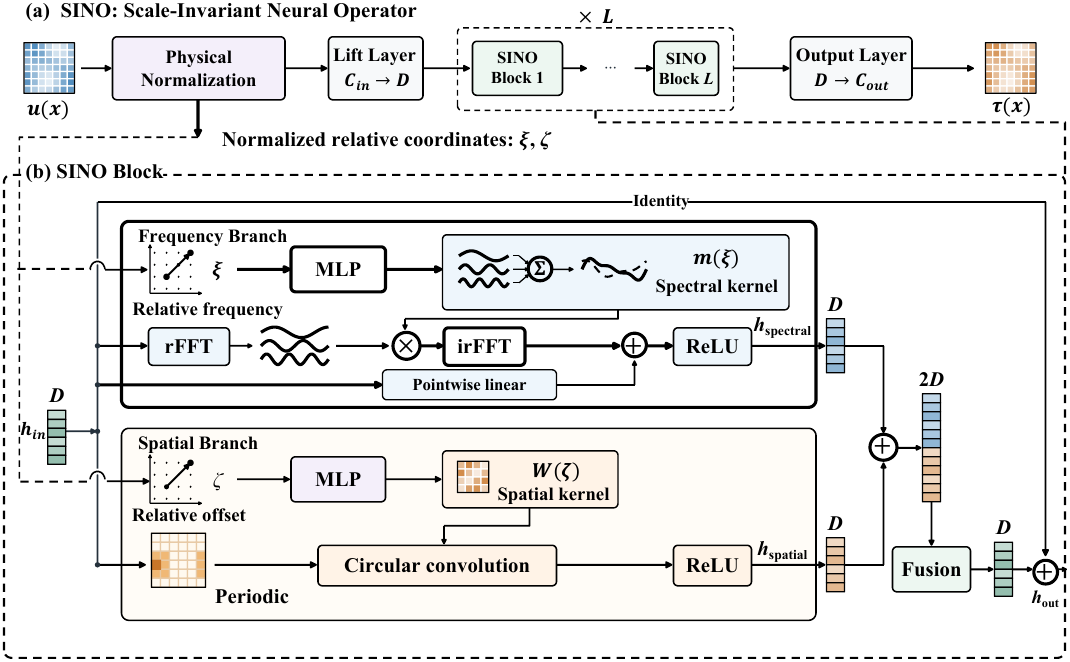}
	\caption{Architecture of the SINO. \textbf{(a)} Overall pipeline with physical normalization, lifting layer, $L$ SINO blocks, and output projection. \textbf{(b)} SINO block: dual-branch architecture where frequency and spatial branches generate continuous kernels via bottleneck MLPs from normalized coordinates $(\xi, \zeta)$, then fuse outputs through residual connections.}
	\label{fig:sino-framework}
\end{figure}

\section{Method}
We propose the Scale-Invariant Neural Operator (SINO), a dual-branch architecture that learns closure terms on normalized physical scales. SINO parameterizes kernel weights as continuous functions of normalized coordinates in both frequency and spatial domains, decoupling learnable parameters from kernel size and resolution. This enables compact representations of scale-invariant physical laws while balancing global spectral structure with local spatial features.

\subsection{Physical Normalization of Input Coordinates}
Learning scale-invariant laws requires that the same physical process, when observed at different resolutions, receives identical normalized representations. For a coarse-grid field $\bar{u}_\Delta$ with resolution $N$ obtained by coarse-graining from fine-grid DNS at resolution $N_{\text{fine}}$, the resample factor $r = N_{\text{fine}}/N$ quantifies the degree of coarse-graining.

\textbf{Frequency domain.} We normalize wavenumber coordinates as $\xi = 4k/N - 1 \in [-1, 1]$, where $k \in [0, N/2]$ is the frequency index from the real FFT. This maps the DC component ($k=0$) to $\xi=-1$, mid-range frequencies to $\xi=0$, and the Nyquist frequency to $\xi=1$, ensuring the same physical frequency receives the same $\xi$ regardless of resolution $N$.

\textbf{Spatial domain.} For a kernel of size $K = 2R + 1$ (radius $R$) at resolution $N$, we normalize the spatial offset $\Delta x = j / N$ (where $j \in [-R, R]$) by the reference length $L_{\text{ref}} = R / (N_{\text{fine}} / r)$, yielding $\zeta = j/R \in [-1, 1]$. This guarantees identical physical displacements receive identical $\zeta$ across resolutions, enabling resolution-independent convolution kernels.

\subsection{Dual-Branch Architecture}

Physical fields in turbulent systems contain both global coherent structures and localized intermittent events. Frequency representations naturally capture smooth long-range correlations and spectral energy transfer but struggle with sharp local features (shocks, vortex cores) due to spectral ringing. Conversely, spatial convolutions excel at localized nonlinear patterns but incur high cost for long-range dependencies. SINO addresses this complementary trade-off via a dual-branch architecture processing features in parallel through frequency and spatial domains.

For input $\bar{u}_\Delta \in \mathbb{R}^{B \times N \times C_\text{in}}$ (batch size $B$, channels $C_\text{in}$), a lifting layer projects to latent space $h^{(0)} = P(\bar{u}_\Delta) \in \mathbb{R}^{B \times N \times D}$ (hidden dimension $D$). Each SINO block ($\ell \in \{1, \ldots, L\}$) processes $h^{(\ell-1)}$ through two branches:

\textbf{Frequency branch} (spectral convolution with residual bypass):
\begin{equation}
h_\text{freq}^{(\ell)} = \sigma\big( W_\text{bypass}^{(\ell)} h^{(\ell-1)} + \mathcal{F}^{-1}\big[ m^{(\ell)}(\xi) \odot \mathcal{F}(h^{(\ell-1)}) \big] \big),
\end{equation}
where $\sigma$ is ReLU, $W_\text{bypass}^{(\ell)} \in \mathbb{R}^{D \times D}$ is the bypass projection, $\mathcal{F}$ denotes FFT, $m^{(\ell)}(\xi) \in \mathbb{C}^{D \times D}$ are hypernetwork-generated multiplication weights, and $\odot$ is element-wise multiplication.

\textbf{Spatial branch} (implicit circular convolution):
\begin{equation}
	h_\text{spatial}^{(\ell)} = \sigma\big( \text{Conv}^{(\ell)}(h^{(\ell-1)}; W^{(\ell)}(\zeta)) \big),
\end{equation}
where $W^{(\ell)}(\zeta)$ are hypernetwork-generated kernel weights parameterized by $\zeta$.

\textbf{Fusion.} Outputs are concatenated and fused:
\begin{equation}
h^{(\ell)} = h^{(\ell-1)} + W_\text{fuse}^{(\ell)} [h_\text{freq}^{(\ell)} \, ; \, h_\text{spatial}^{(\ell)}],
\end{equation}
where $[\cdot \, ; \, \cdot]$ denotes concatenation and $W_\text{fuse}^{(\ell)} \in \mathbb{R}^{D \times 2D}$ is the fusion matrix. Activations are applied in intermediate blocks but omitted in the final block for unrestricted output range.

After $L$ blocks, an output projection with small initialization ($\sigma_\text{init} = 0.01$) yields:
\begin{equation}
	\tau_\Delta = Q(h^{(L)}) \in \mathbb{R}^{B \times N \times C_\text{out}}.
\end{equation}
The predicted closure term $\tau_\Delta$ is then incorporated into the coarse-grid evolution $\partial \bar{u}_\Delta / \partial t = \mathcal{L}_\Delta(\bar{u}_\Delta) + \tau_\Delta$ for rollout integration.

\subsection{Implicit Kernel Representation via Hypernetworks}

We represent convolution kernels as continuous functions generated by small MLPs (hypernetworks), decoupling expressiveness from parameter count.

\textbf{Frequency domain.} Hypernetwork $\Phi_\omega$ generates complex-valued weights $m(\xi) \in \mathbb{C}^{D \times D}$ via three hidden layers $[w_f, w_f, n_f]$ (intermediate width $w_f$, bottleneck $n_f$) with separate real/imaginary output heads:
\begin{equation}
m(\xi) = \Phi_\omega(\xi) = \text{head}_\text{real}(\text{MLP}(\xi)) + i \cdot \text{head}_\text{imag}(\text{MLP}(\xi)).
\end{equation}
Multiplication applies via the convolution theorem: $\mathcal{F}(\text{Conv}(h)) = m(\xi) \odot \mathcal{F}(h)$, enabling $O(N \log N)$ computation via FFT.

\textbf{Spatial domain.} Hypernetwork $\Phi_x$ generates kernel weights $W(\zeta) \in \mathbb{R}^{K \times K \times D \times D}$ via three layers $[w_s, w_s, n_s]$ (spatial width $w_s$, bottleneck $n_s$):
\begin{equation}
W(\zeta) = \Phi_x(\zeta), \quad (\text{Conv}(h))_i = \sum_{j \in \mathcal{N}(i)} W(\zeta_{i-j}) \cdot h_j,
\end{equation}
where $\mathcal{N}(i)$ is the neighborhood around position $i$ with circular boundary conditions. The implicit representation induces spatial smoothness, improving robustness.The complete model is:
\begin{equation}
\tau_\Delta = \text{SINO}(\bar{u}_\Delta; \Theta) = Q \circ \text{Block}^{(L)} \circ \cdots \circ \text{Block}^{(1)} \circ P(\bar{u}_\Delta),
\end{equation}
where $\Theta$ collects all parameters. Hyperparameters ($D, L, K, n_f, n_s, w_f, w_s$) control capacity and inductive biases.

\subsection{Theoretical Properties}

SINO's implicit kernel representation induces low-rank structure through bottleneck layers. To validate this empirically, we perform Principal Component Analysis (PCA) on the learned convolution kernels across all channel pairs in each layer. Figure~\ref{fig:cumulative_variance} reveals a striking concentration of variance: the first 2 principal components alone capture over 90\% of the variance in both spatial and frequency branches, while merely 3 components exceed 95\% explained variance across all layers. This demonstrates that SINO successfully concentrates learned representations into an extremely low-dimensional subspace, aligned with the physical intuition of scale separation and energy cascade in turbulent flows. We rigorously establish in Appendix~\ref{sec:theory} that the bottleneck architecture enforces explicit rank constraints (Theorem~\ref{thm:explicit_rank_constraint}), where every generated kernel admits a finite-rank decomposition $m(\xi) = \sum_{j=1}^{n_f} z_j(\xi) \cdot B_j$ with fixed basis matrices, reducing effective parameter count from $O(k_{\max} \cdot D^2)$ in FNO or $O(K^2 \cdot D^2)$ in CNN to $O((n_f + n_s) \cdot D^2)$ in SINO, independent of resolution or kernel size. Furthermore, the MLP parameterization guarantees Lipschitz continuity (Theorem~\ref{thm:lipschitz_hypernetworks}) that prevents memorization of discrete grid patterns, and we establish quantitative operator norm bounds: $O(\sqrt{n_f})$ sublinear growth for the frequency-domain operator and $O(n_s)$ linear growth for the spatial-domain operator (Theorems~\ref{thm:frequency_domain_operator_norm}--\ref{thm:spatial_domain_operator_norm}), providing favorable implicit regularization through spectral orthogonality in the frequency branch. These properties jointly explain SINO's superior parameter efficiency and data efficiency observed in experiments.

\begin{figure}[H]
\centering
\includegraphics[width=0.95\textwidth]{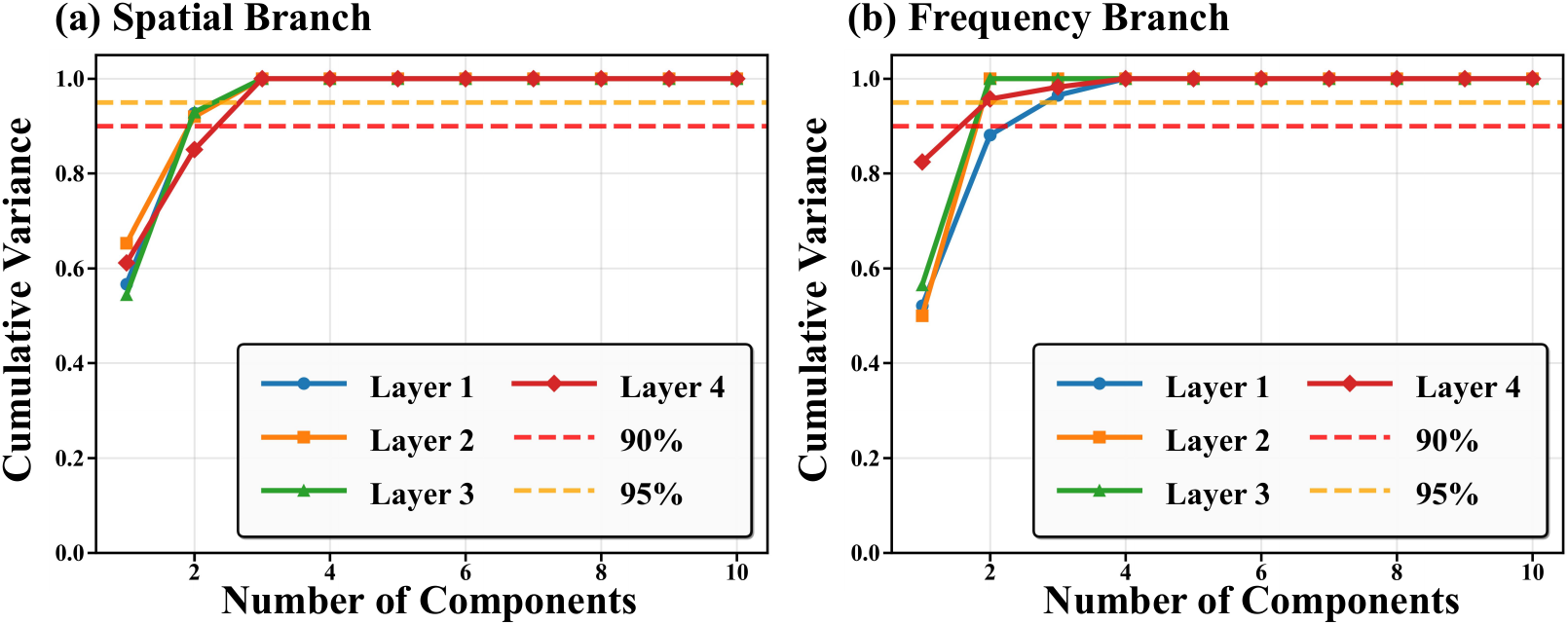}
\caption{Cumulative variance explained by principal components of learned convolution kernels on Forcing NS (Re=4000). \textbf{(a) Spatial branch} and \textbf{(b) Frequency branch} across 4 layers. Remarkably, the first 2 components capture $>$90\% variance and 3 components exceed $>$95\% (horizontal dashed lines), validating the extremely low-rank structure induced by SINO's bottleneck architecture and confirming that the model prioritizes dominant physical modes over high-dimensional noise.}
\label{fig:cumulative_variance}
\end{figure}

\section{Experiments}
We validate SINO on turbulent closure problems spanning 1D and 2D systems, comparing against diverse baselines and conducting ablation studies to verify effectiveness.

\subsection{Experimental Setup}
\paragraph{Benchmarks.}
We evaluate SINO on six turbulent closure benchmarks with aggressive coarse-graining ratios. Table~\ref{tab:benchmarks} summarizes the configurations, including forcing and decaying variants of Burgers turbulence, Kuramoto-Sivashinsky (KS) turbulence, and Navier-Stokes (NS) turbulence at different Reynolds numbers. The forcing variants use external forcing (Kolmogorov forcing for NS), while the decaying variants exhibit freely decaying dynamics, testing the model's ability to capture diverse physical regimes. Supplementary experiments at alternative resolutions and forcing configurations (Appendix~\ref{sec:additional_experiments}) further validate SINO's scale-invariant capability. Detailed numerical schemes and simulation protocols are in Appendix~\ref{sec:benchmark_equations}.

\begin{table}[H]
\centering
\caption{Benchmark configurations.}
\label{tab:benchmarks}
\small
\begin{tabular}{lcccc}
\toprule
\textbf{Problem} & \textbf{Dim} & \textbf{Resolution} & \textbf{Spatial Discretization} & \textbf{Temporal Scheme} \\
\midrule
Forcing Burgers & 1D & 512$\to$32 & WENO5 Finite Volume Method & SSP-RK3 \\
Decaying Burgers & 1D & 2048$\to$256 & Pseudo-Spectral & TVD RK3 \\
KS & 1D & 256$\to$64 & Pseudo-Spectral & ETDRK4 \\
Forcing NS (Re=1000) & 2D & 512$\to$64 & Finite Volume Method (Van Leer) & Semi-Implicit \\
Forcing NS (Re=4000) & 2D & 512$\to$64 & Finite Volume Method (Van Leer) & Semi-Implicit \\
Decaying NS & 2D & 512$\to$64 & Finite Volume Method (Van Leer) & Semi-Implicit \\
\bottomrule
\end{tabular}
\end{table}

\paragraph{Baselines.}
We compare SINO against nine baselines covering three paradigms: traditional models (U-Net, DeepONet), Transformer-based operators (Transolver, Oformer, GK-Transformer), and frequency-domain operators (FNO, AMFNO, UFNO), plus uncorrected coarse-grid DNS as a physical baseline. Model specifications and training protocols are in Appendices~\ref{sec:model_specs} and~\ref{sec:training_methodology}.

\begin{figure}[H]
	\centering
	\includegraphics[width=0.9\textwidth]{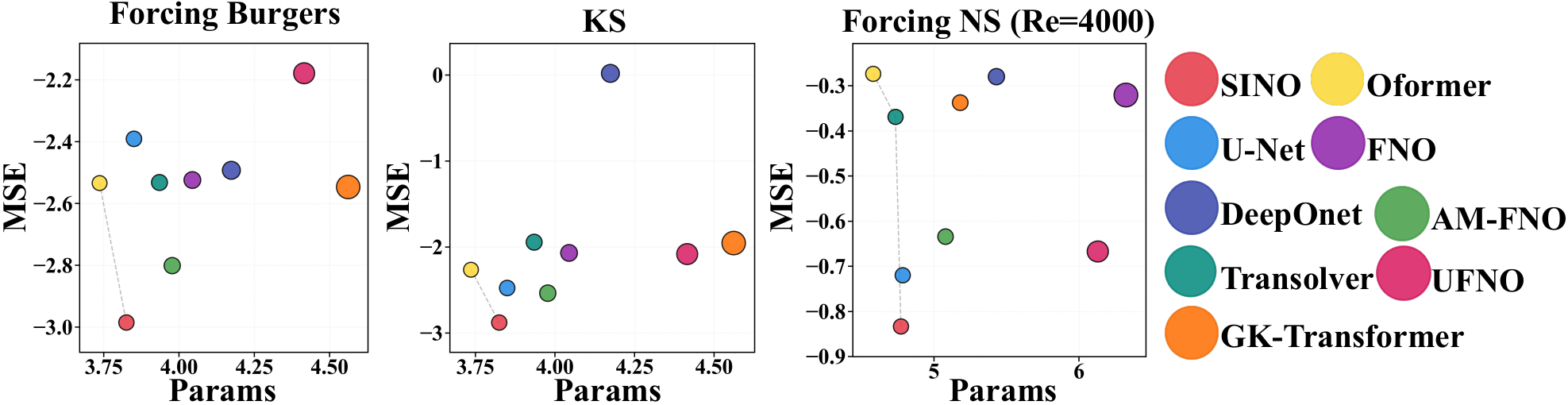}
	\caption{Pareto comparison of MSE versus parameter count on Forcing Burgers, KS, and Forcing NS (Re=4000). Both axes are on a logarithmic scale. SINO (red) achieves optimal or near-optimal MSE with significantly fewer parameters than FNO-based methods, demonstrating parameter efficiency from implicit low-rank representations.}
	\label{fig:pareto_comparison}
\end{figure}

\subsection{Main Results}
Tables~\ref{tab:results} and~\ref{tab:results_2d} present quantitative comparisons on 1D and 2D turbulence benchmarks. SINO consistently achieves superior or competitive performance across all tasks while maintaining parameter efficiency. Supplementary experiments on additional resolutions and forcing configurations (Tables~\ref{tab:results_supplementary_1d} and~\ref{tab:results_supplementary_2d} in Appendix) confirm robustness across diverse physical regimes. Following the Indirect Neural Corrector (INC) framework~\citep{wei2026inc}, SINO predicts closure terms as right-hand side corrections; training details are in Section~\ref{sec:training_methodology}. All reported metrics are extrapolation errors: 30\% of time steps for training, 70\% for testing.

\begin{figure}[htbp]
	\centering
	\includegraphics[width=0.88\textwidth]{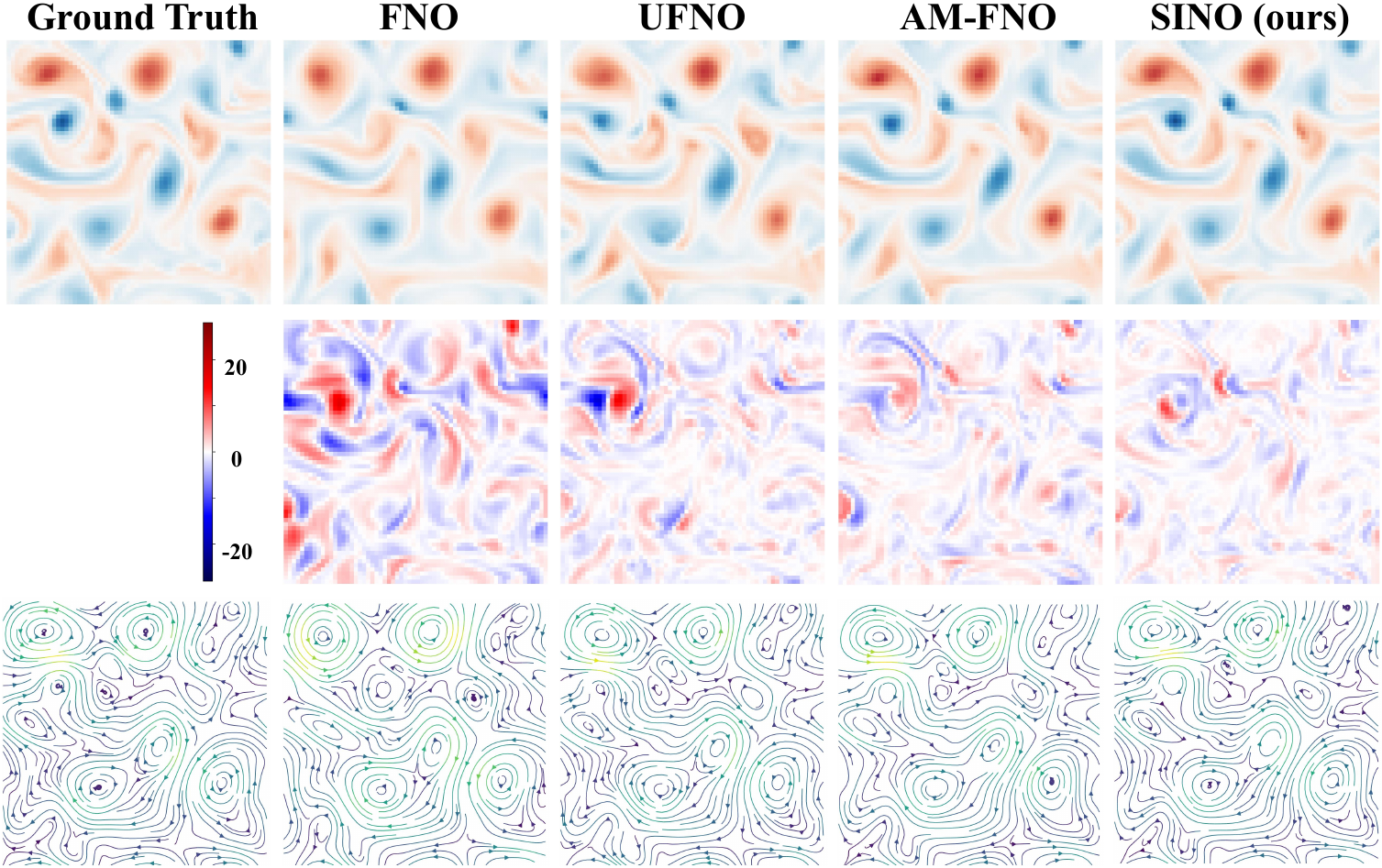}
	\caption{Instantaneous predictions on Kolmogorov-forced Navier--Stokes turbulence at Re=4000. Column 1: ground truth downsampled from DNS-512 to 64$\times$64 resolution. Columns 2--5: FNO, UFNO, AMFNO, SINO predictions. Row 1: vorticity fields. Row 2: vorticity error distributions. Row 3: streamline plots. SINO accurately reconstructs fine-scale structures and velocity patterns with minimal error, while other methods exhibit varying degrees of distortion in small-scale details.}
	\label{fig:field}
\end{figure}

\textbf{1D Benchmarks.} On Forcing Burgers, SINO achieves the lowest error (1.03E-03 average, 9.41E-04 best) with only 6.7K parameters, outperforming second-best AMFNO (1.58E-03 average, 9.5K parameters) by 1.5$\times$. Transformer-based methods yield comparable parameter counts but higher errors: GK-Transformer (2.84E-03 average, 36K parameters), Oformer (2.92E-03, 5.4K parameters), and Transolver (2.94E-03, 8.6K parameters). FNO variants show mixed results: FNO (2.99E-03, 11K parameters), DeepONet (3.22E-03, 15K parameters), U-Net (4.07E-03, 7.1K parameters), and UFNO (6.63E-03, 26K parameters) all lag behind SINO. On KS turbulence, SINO (1.33E-03 average, 1.14E-03 best) surpasses all baselines by substantial margins: 2.2$\times$ better than AMFNO (2.92E-03), 2.5$\times$ better than U-Net (3.33E-03), and 6--8$\times$ better than FNO variants and Transformers. For Decaying Burgers, SINO (4.37E-03 average, 2.90E-03 best) achieves 38$\times$ improvement over most baselines, which plateau around 1.65E-01. Only SINO-SPA (2.09E-02 average, 1.8K parameters) approaches SINO's performance among ablation variants.

\begin{table}[H]
	\centering
	\caption{Performance comparison on 1D benchmarks.}
	\label{tab:results}
	\footnotesize
	\begin{tabular*}{\textwidth}{@{\extracolsep{\fill}}lccccccr@{}}
		\toprule
		\multirow{3}{*}{\textbf{Method}} & \multicolumn{6}{c}{\textbf{1D Benchmarks}} & \multirow{3}{*}{\textbf{Params}} \\
		\cmidrule(lr){2-7}
		& \multicolumn{2}{c}{\textbf{Forcing Burgers}} & \multicolumn{2}{c}{\textbf{KS}} & \multicolumn{2}{c}{\textbf{Decaying Burgers}} & \\
		\cmidrule(lr){2-3} \cmidrule(lr){4-5} \cmidrule(lr){6-7}
		& \textbf{AVE} & \textbf{Best} & \textbf{AVE} & \textbf{Best} & \textbf{AVE} & \textbf{Best} & \\
		\midrule
		U-Net & 4.07E-03 & 3.31E-03 & 3.33E-03 & 1.48E-03 & 1.65E-01 & 1.65E-01 & 7,077 \\
		DeepONet & 3.22E-03 & 2.78E-03 & 1.05E+00 & 1.39E+00 & 1.65E-01 & 1.65E-01 & 14,913 \\
		Transolver & 2.94E-03 & 2.80E-03 & 1.14E-02 & 1.09E-02 & 1.65E-01 & 1.64E-01 & 8,601 \\
		Oformer & 2.92E-03 & 2.92E-03 & 5.45E-03 & 5.45E-03 & 1.65E-01 & 1.65E-01 & 5,444 \\
		GK-Transformer & 2.84E-03 & 2.77E-03 & 1.12E-02 & 1.11E-02 & 1.65E-01 & 1.64E-01 & 36,449 \\
		FNO & 2.99E-03 & 2.82E-03 & 8.55E-03 & 8.03E-03 & 1.65E-01 & 1.65E-01 & 11,073 \\
		UFNO & 6.63E-03 & 5.70E-03 & 8.28E-03 & 7.82E-03 & 1.65E-01 & 1.65E-01 & 26,017 \\
		AM-FNO & 1.58E-03 & 1.34E-03 & 2.92E-03 & 2.50E-03 & 1.66E-01 & 1.65E-01 & 9,489 \\
		\midrule
		SINO-SPE & 1.56E-03 & 1.32E-03 & 3.02E-03 & 2.49E-03 & 1.28E-01 & 1.27E-01 & 3,877 \\
		SINO-SPA & 2.12E-03 & 1.79E-03 & 3.53E-03 & 2.11E-03 & 2.09E-02 & 4.08E-03 & 1,797 \\
		\midrule
		SINO (ours) & \textbf{1.03E-03} & \textbf{9.41E-04} & \textbf{1.33E-03} & \textbf{1.14E-03} & \textbf{4.37E-03} & \textbf{2.90E-03} & 6,681 \\
		\midrule
		Coarse Grid & \multicolumn{2}{c}{2.23E-01} & \multicolumn{2}{c}{5.46E-03} & \multicolumn{2}{c}{1.65E-01} & -- \\
		\bottomrule
	\end{tabular*}
\end{table}

\begin{table}[htbp]
	\centering
	\caption{Performance comparison on 2D benchmarks.}
	\label{tab:results_2d}
	\footnotesize
	\setlength{\tabcolsep}{2.8pt}
	\begin{tabular*}{\textwidth}{@{\extracolsep{\fill}}lccccccr@{}}
		\toprule
		\multirow{3}{*}{\textbf{Method}} & \multicolumn{6}{c}{\textbf{2D Benchmarks}} & \multirow{3}{*}{\textbf{Params}} \\
		\cmidrule(lr){2-7}
		& \multicolumn{2}{c}{\textbf{Forcing NS (Re=1000)}} & \multicolumn{2}{c}{\textbf{Forcing NS (Re=4000)}} & \multicolumn{2}{c}{\textbf{Decaying NS}} & \\
		\cmidrule(lr){2-3} \cmidrule(lr){4-5} \cmidrule(lr){6-7}
		& \textbf{AVE} & \textbf{Best} & \textbf{AVE} & \textbf{Best} & \textbf{AVE} & \textbf{Best} & \\
		\midrule
		U-Net & 1.34E-01 & 1.16E-01 & 1.91E-01 & 1.84E-01 & 6.67E-02 & 5.64E-02 & 61,006 \\
		DeepONet & 4.62E-01 & 4.60E-01 & 5.25E-01 & 5.18E-01 & 2.95E-01 & 2.79E-01 & 269,153 \\
		Transolver & 3.65E-01 & 3.48E-01 & 4.27E-01 & 4.23E-01 & 2.24E-01 & 2.17E-01 & 54,306 \\
		Oformer & 5.07E-01 & 5.07E-01 & 5.33E-01 & 5.33E-01 & 4.03E-01 & 4.03E-01 & 38,087 \\
		GK-Transformer & 3.87E-01 & 3.84E-01 & 4.60E-01 & 4.55E-01 & 2.04E-01 & 1.96E-01 & 151,874 \\
		FNO & 4.44E-01 & 4.42E-01 & 4.78E-01 & 4.75E-01 & 2.60E-01 & 2.40E-01 & 2,106,018 \\
		UFNO & 1.64E-01 & 1.51E-01 & 2.15E-01 & 2.12E-01 & \textbf{5.58E-02} & 4.86E-02 & 1,346,466 \\
		AM-FNO & 1.69E-01 & 1.59E-01 & 2.32E-01 & 2.15E-01 & 7.68E-02 & 7.05E-02 & 120,162 \\
		\midrule
		SINO-SPE & 2.18E-01 & 1.81E-01 & 2.63E-01 & 2.40E-01 & 1.71E-01 & 1.51E-01 & 33,834 \\
		SINO-SPA & 2.23E-01 & 2.20E-01 & 2.82E-01 & 2.72E-01 & 1.37E-01 & 1.24E-01 & 17,322 \\
		\midrule
		SINO (ours) & \textbf{9.17E-02} & \textbf{8.43E-02} & \textbf{1.47E-01} & \textbf{1.41E-01} & 5.99E-02 & \textbf{4.55E-02} & 59,314 \\
		\midrule
		Coarse Grid & \multicolumn{2}{c}{5.07E-01} & \multicolumn{2}{c}{5.32E-01} & \multicolumn{2}{c}{4.04E-01} & -- \\
		\bottomrule
	\end{tabular*}
\end{table}

\textbf{2D Benchmarks.} On Forcing NS (Re=1000), SINO (9.17E-02 average, 8.43E-02 best, 59K parameters) leads all methods. U-Net (1.34E-01, 61K parameters) achieves second place with comparable parameter count but 1.5$\times$ higher error. FNO-based methods underperform: UFNO (1.64E-01, 1.3M parameters) and AMFNO (1.69E-01, 120K parameters) require 2--23$\times$ more parameters for inferior accuracy. Transformer methods exhibit substantially higher errors: Transolver (3.65E-01, 54K parameters), GK-Transformer (3.87E-01, 152K parameters), and Oformer (5.07E-01, 38K parameters). FNO (4.44E-01, 2.1M parameters) and DeepONet (4.62E-01, 269K parameters) underperform despite massive capacity. At Re=4000, similar trends emerge: SINO (1.47E-01 average, 1.41E-01 best) maintains first place, followed by U-Net (1.91E-01), UFNO (2.15E-01), and AMFNO (2.32E-01). On Decaying NS, UFNO achieves the lowest average error (5.58E-02) but requires 1.3M parameters—23$\times$ more than SINO (5.99E-02 average, 59K parameters). Notably, SINO attains the best peak performance (4.55E-02) across all methods, demonstrating superior optimization capability. U-Net (6.67E-02, 61K parameters) ranks third, while AMFNO (7.68E-02, 120K parameters) requires 2$\times$ more parameters for worse accuracy.

\textbf{Ablation Analysis.} SINO-SPE (frequency-only) and SINO-SPA (spatial-only) consistently underperform full SINO by 1.5--3$\times$ across all benchmarks, confirming the necessity of dual-branch architecture. SINO-SPE shows particular degradation on 2D problems (2.18E-01 vs. 9.17E-02 on Forcing NS Re=1000), while SINO-SPA struggles on forced problems requiring global energy transfer. The full model synergistically combines both branches for optimal performance.

\subsection{Scaling Laws and Cross-Scale Analysis}

\textbf{Parameter Efficiency.} Figure~\ref{fig:scaling} reveals SINO's superior scaling behavior through power-law exponents. On forced NS (Re=4000), SINO achieves $\alpha=0.568$—nearly 38$\times$ steeper than FNO ($\alpha=0.015$) and 3--4$\times$ steeper than AMFNO ($\alpha=0.143$) and UFNO ($\alpha=0.162$). On decaying NS, SINO maintains $\alpha=0.868$ compared to FNO's $\alpha=0.073$, demonstrating 12$\times$ better parameter efficiency. These steep slopes indicate that SINO effectively learns low-rank physical structures: each additional parameter contributes substantially to error reduction, whereas baselines require orders of magnitude more capacity for comparable gains.

\begin{figure}[htbp]
\centering
\includegraphics[width=0.88\textwidth]{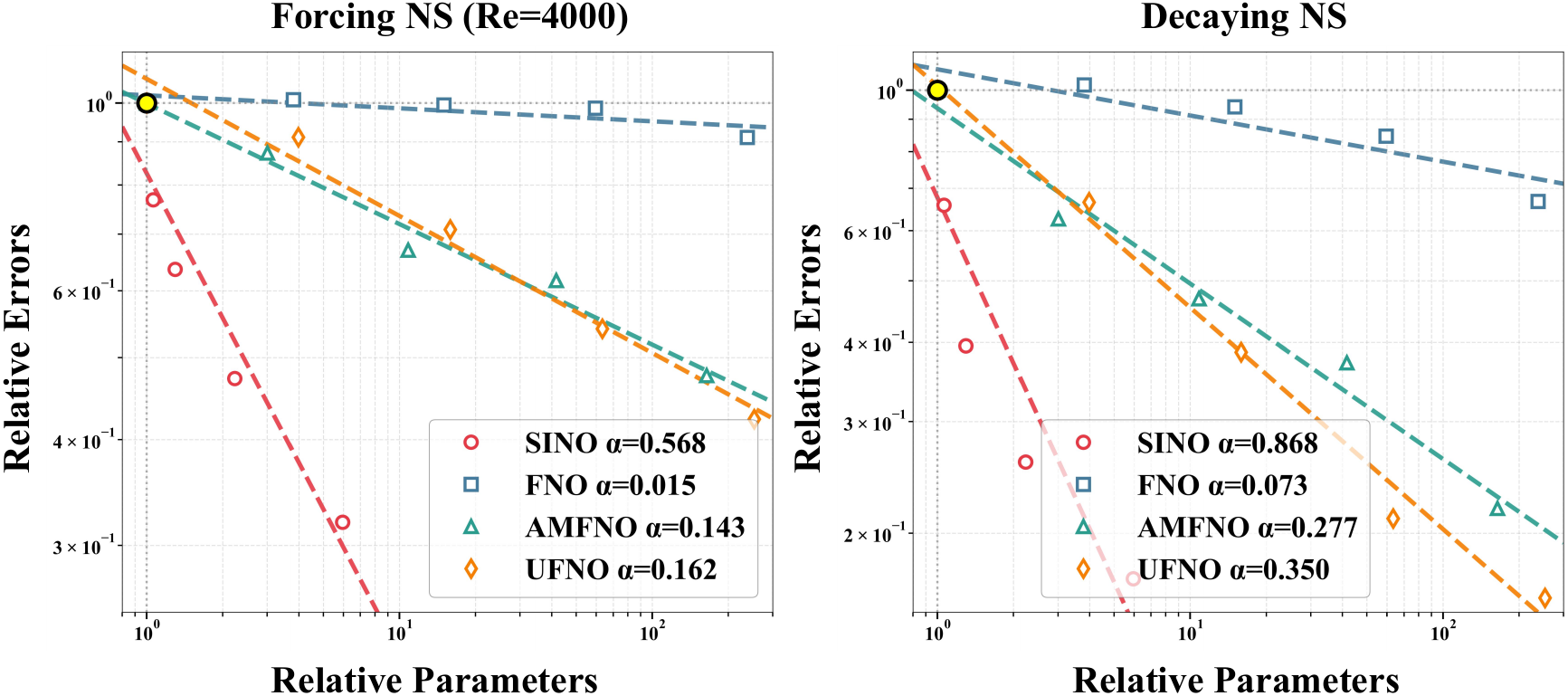}
\caption{Scaling laws on forced NS (Re=4000, left) and decaying NS (right). SINO exhibits steeper power-law exponents ($\alpha=0.568$ and $0.868$) than baselines, demonstrating better parameter efficiency.}
\label{fig:scaling}
\end{figure}

\begin{figure}[htbp]
	\centering
	\includegraphics[width=0.88\textwidth]{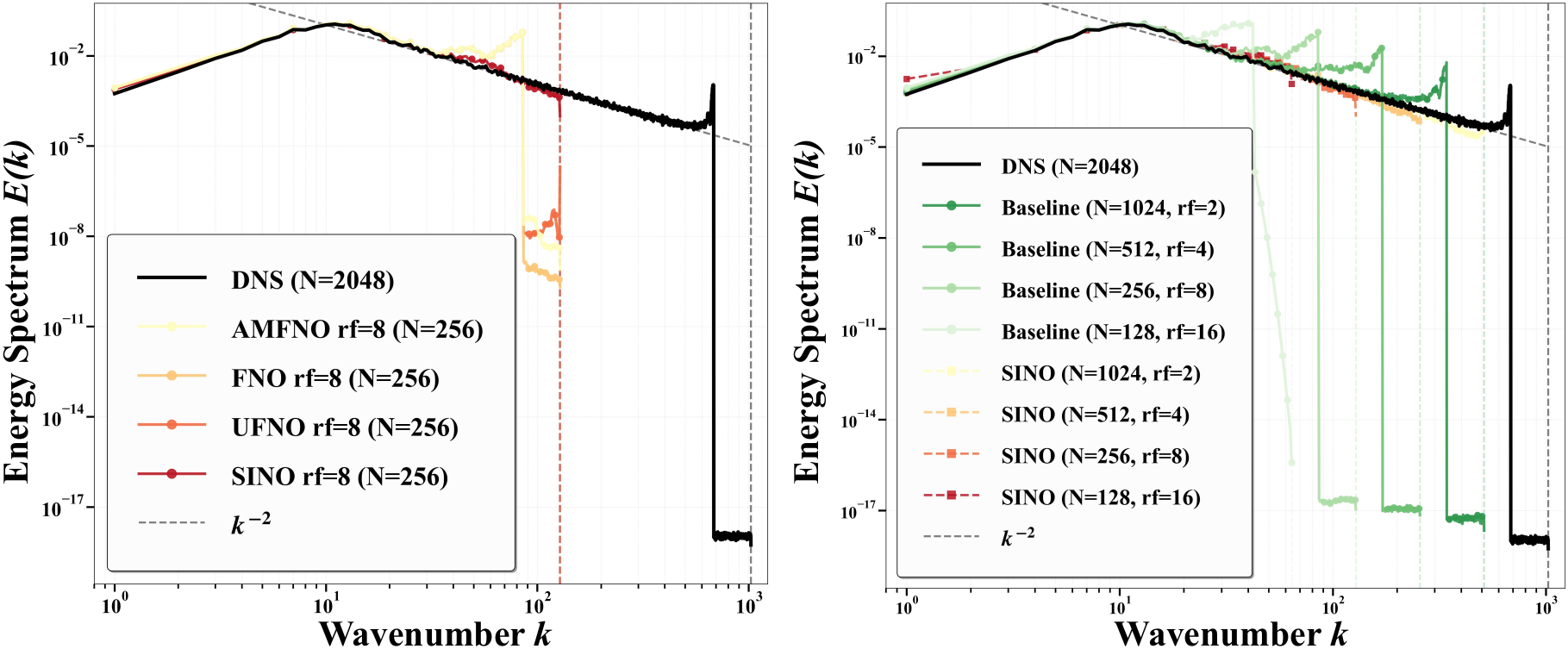}
	\caption{Energy spectra on decaying Burgers turbulence. \textbf{Left:} At resample factor (rf) = 8, SINO most accurately follows the $k^{-2}$ inertial range among neural operators. \textbf{Right:} Across rf = 2--16, baselines (green, solid) exhibit numerical dissipation beyond cutoffs, while SINO (red-yellow, dashed) preserves correct spectral decay, validating scale-invariant modeling.}
	\label{fig:spectra}
\end{figure}

\textbf{Spectral Fidelity.} Figure~\ref{fig:spectra} validates SINO's cross-scale capability through energy spectra on decaying Burgers turbulence. Traditional coarse grids exhibit severe numerical dissipation beyond the Nyquist cutoff (left panel, vertical dashed lines mark resolution limits). At rf = 8 (N = 256), baseline methods (AMFNO, FNO, UFNO) deviate from the correct $k^{-2}$ decay slope and fail to suppress spurious dissipation (middle panel). In contrast, SINO maintains accurate spectral behavior across resolutions: even at extreme downsampling (rf = 16, N = 128), SINO's energy spectrum aligns with high-resolution DNS (N = 2048) and preserves the theoretical $k^{-2}$ slope (right panel), confirming its scale-invariant representations.

\textbf{Bottleneck Capacity Scaling.} To investigate the impact of bottleneck dimensions on performance, Table~\ref{tab:bottleneck_scaling} compares SINO variants with $(n_f, n_s)$ ranging from 2 to 32 against uncorrected DNS at multiple resolutions. SINO-32 achieves errors of 4.08E-02 (Decaying NS) and 6.06E-02 (Forcing NS Re=1000), representing 1.5--1.7$\times$ improvement over SINO-2 and  exceeding the accuracy of DNS-256 (4$\times$ higher resolution with 16$\times$ more grid points). Figure~\ref{fig:error_time} confirms SINO-32 maintains superior temporal stability in long-term extrapolation. This demonstrates that increasing bottleneck capacity enables SINO to surpass even higher-resolution coarse-grid simulations, validating the effectiveness of further parameter investment in low-rank subspace learning.

\section{Conclusion}
We propose the Scale-Invariant Neural Operator (SINO), a dual-branch architecture that learns turbulent closure terms on normalized physical scales. By parameterizing convolution kernels as continuous functions via bottleneck hypernetworks, SINO decouples model capacity from grid resolution, enabling compact representation of scale-invariant physical laws. The bottleneck structure enforces explicit low-rank constraints, concentrating $>$95\% variance in merely 2-3 modes (validated by PCA) and forcing the model to discover dominant physical mechanisms rather than memorize high-dimensional patterns. This yields 38$\times$ steeper scaling law exponents than FNO, demonstrating superior parameter efficiency. Experiments on six turbulent closure benchmarks demonstrate 1.5--38$\times$ error reduction with 2--23$\times$ parameter efficiency over baselines. Ablation studies confirm dual-branch necessity, and spectral analysis validates cross-scale generalization under extreme downsampling. Future work will extend SINO to three-dimensional flows and diverse physical problems, establishing SINO as a general framework for learning scale-invariant laws where low-rank structure governs system dynamics.

\section*{Acknowledgments}
We thank Tao Zhang, Ruiqi Feng, and Xinan Dai from Westlake University for discussions and for providing feedback on our manuscript. We also gratefully acknowledge the support of the Westlake University Research Center for Industries of the Future and the Westlake University Center for High-performance Computing. The content is solely the responsibility of the authors and does not necessarily represent the official views of the funding entities. This work was supported by the Shanghai Municipal-Level Major Special Project.

\bibliography{iclr2027_conference}
\bibliographystyle{iclr2027_conference}

\clearpage

\appendix
\section{Appendix}

This appendix provides comprehensive technical details supporting the main results. Section~\ref{sec:benchmark_equations} specifies the governing equations, numerical discretization schemes, and simulation protocols for all six turbulent closure benchmarks. Section~\ref{sec:model_specs} presents complete architectural specifications for SINO and all baseline models across 1D and 2D problems. Section~\ref{sec:training_methodology} documents the training objectives, rollout integration with closure term incorporation, curriculum learning strategy, optimization settings, and time integration protocols. Section~\ref{sec:theory} establishes theoretical foundations through formal proofs of SINO's low-rank structure, Lipschitz regularity, and operator norm bounds. Finally, Section~\ref{sec:additional_experiments} reports additional experimental results including cross-resolution and alternative forcing generalization tests, principal component analysis confirming low-rank structure across different flow regimes, and bottleneck dimension scaling studies demonstrating computational trade-offs.

\subsection{Benchmark Equations}
\label{sec:benchmark_equations}

\subsubsection{Forcing Burgers Equation}

\paragraph{Governing Equation.}
We consider the one-dimensional viscous Burgers equation in conservative form with external forcing:
\begin{equation}
\frac{\partial u}{\partial t} + \frac{\partial}{\partial x}\left(\frac{u^2}{2}\right) = \eta \frac{\partial^2 u}{\partial x^2} + f(x,t), \quad x \in [0, L], \quad t > 0,
\end{equation}
where $u(x,t)$ is the velocity field, $\eta = 0.01$ is the kinematic viscosity, and $f(x,t)$ is the external forcing term. The domain length is $L = 2\pi$ with periodic boundary conditions.

\paragraph{Initial and Boundary Conditions.}
The initial condition is set to zero velocity everywhere: $u(x, 0) = 0$. Periodic boundary conditions are enforced: $u(0, t) = u(L, t)$ for all $t \geq 0$. A warmup phase of $T_{\text{warmup}} = 2.0$ is applied to allow shock structures to develop before recording snapshots. The external forcing is given by:
\begin{equation}
f(x,t) = \sum_{i=1}^{n_{\text{modes}}} A_i \sin\left(\omega_i t + \frac{2\pi \ell_i x}{L} + \phi_i\right),
\end{equation}
where $n_{\text{modes}} = 20$, and for each mode $i$: $A_i \sim \mathcal{U}(-0.5, 0.5)$, $\omega_i \sim \mathcal{U}(-0.4, 0.4)$, $\ell_i \sim \mathcal{U}_{\text{discrete}}\{3, 4, 5, 6\}$, and $\phi_i \sim \mathcal{U}(0, 2\pi)$ are randomly sampled once per trajectory and held constant throughout the simulation.

\paragraph{Spatial Discretization.}
The equation is discretized using a finite volume method with fifth-order Weighted Essentially Non-Oscillatory (WENO5) reconstruction on a uniform grid of $N = 512$ cells with cell size $\Delta x = L/N$. Let $\bar{u}_i(t)$ denote the cell-averaged value in cell $i$. The WENO5 reconstruction computes left and right interface values $u_{i+1/2}^-$ and $u_{i+1/2}^+$ from a five-point stencil $\{\bar{u}_{i-2}, \bar{u}_{i-1}, \bar{u}_i, \bar{u}_{i+1}, \bar{u}_{i+2}\}$ via:
\begin{equation}
u_{i+1/2}^- = \sum_{k=0}^{2} \omega_k^- p_k^-(u_{i-2+k:i+2+k}), \quad u_{i+1/2}^+ = \sum_{k=0}^{2} \omega_k^+ p_k^+(u_{i-1-k:i+3-k}),
\end{equation}
where $p_k^{\pm}$ are third-order polynomial interpolants and $\omega_k^{\pm}$ are nonlinear weights computed from smoothness indicators:
\begin{equation}
\omega_k = \frac{\alpha_k}{\sum_{j=0}^{2} \alpha_j}, \quad \alpha_k = \frac{d_k}{(\epsilon + \beta_k)^2},
\end{equation}
with optimal linear weights $d_0 = 0.1$, $d_1 = 0.6$, $d_2 = 0.3$, regularization parameter $\epsilon = 10^{-6}$, and smoothness indicators $\beta_k$ measuring local solution variation. The convective flux is computed using the Godunov numerical flux:
\begin{equation}
F_{i+1/2}^{\text{conv}} = \begin{cases}
0, & \text{if } u_{i+1/2}^- \leq 0 \leq u_{i+1/2}^+, \\
\min\left(\frac{(u_{i+1/2}^-)^2}{2}, \frac{(u_{i+1/2}^+)^2}{2}\right), & \text{if } u_{i+1/2}^- \leq u_{i+1/2}^+, \\
\max\left(\frac{(u_{i+1/2}^-)^2}{2}, \frac{(u_{i+1/2}^+)^2}{2}\right), & \text{otherwise}.
\end{cases}
\end{equation}
The diffusive flux is approximated by second-order central differences: $F_{i+1/2}^{\text{diff}} = -\eta (\bar{u}_{i+1} - \bar{u}_i) / \Delta x$. The semi-discrete form becomes:
\begin{equation}
\frac{d\bar{u}_i}{dt} = -\frac{F_{i+1/2} - F_{i-1/2}}{\Delta x} + f_i(t), \quad F_{i+1/2} = F_{i+1/2}^{\text{conv}} + F_{i+1/2}^{\text{diff}}.
\end{equation}

\paragraph{Temporal Discretization.}
Time integration employs the third-order Strong Stability Preserving Runge-Kutta method (SSP-RK3) in Shu-Osher form:
\begin{align}
\bar{u}^{(1)} &= \bar{u}^n + \Delta t \, \mathcal{L}(\bar{u}^n), \\
\bar{u}^{(2)} &= \frac{3}{4} \bar{u}^n + \frac{1}{4} \left[\bar{u}^{(1)} + \Delta t \, \mathcal{L}(\bar{u}^{(1)})\right], \\
\bar{u}^{n+1} &= \frac{1}{3} \bar{u}^n + \frac{2}{3} \left[\bar{u}^{(2)} + \Delta t \, \mathcal{L}(\bar{u}^{(2)})\right],
\end{align}
where $\mathcal{L}(\bar{u})$ denotes the right-hand side spatial operator. The timestep $\Delta t$ is adaptively chosen at each RK stage to satisfy the CFL condition with safety factor $\nu_{\text{CFL}} = 0.4$:
\begin{equation}
\Delta t = \nu_{\text{CFL}} \cdot \min\left(\frac{\Delta x}{\max_i |\bar{u}_i|}, \frac{\Delta x^2}{2\eta}\right).
\end{equation}

\paragraph{Simulation Protocol.}
Each trajectory undergoes a warmup phase of $T_{\text{warmup}} = 2.0$ time units to develop shock structures, followed by a recording phase of $T_{\text{max}} = 10.0$ time units. During the recording phase, $n_{\text{snapshots}} = 1001$ snapshots (including the initial state at $t = 0$ of the recording phase) are saved at uniform intervals $\Delta t_{\text{outer}} = T_{\text{max}} / (n_{\text{snapshots}} - 1) = 0.01$. A total of 64 trajectories are generated with different random forcing parameters. The fine-grid snapshots at resolution $N = 512$ are coarsened to multiple resolutions via exact finite volume averaging. For instance, with resample factor $r = 16$, the snapshots are downsampled to $N = 32$, yielding the coarse-grid reference states $\bar{u}_{\Delta}$ used for closure modeling. The DNS solver employs adaptive timesteps determined by the CFL condition: at each RK substep, the timestep is computed as $\Delta t = \nu_{\text{CFL}} \cdot \min(\Delta x / \max_i |\bar{u}_i|, \, \Delta x^2 / (2\eta))$, where the first term represents the convective stability limit and the second term the diffusive stability limit. The minimum of these two constraints ensures numerical stability throughout the simulation, with the timestep adapting dynamically to the instantaneous solution state.

\subsubsection{Kuramoto--Sivashinsky Equation}

\paragraph{Governing Equation.}
We consider the one-dimensional Kuramoto--Sivashinsky (KS) equation on a periodic domain:
\begin{equation}
\frac{\partial v}{\partial t} + v \frac{\partial v}{\partial x} + \frac{\partial^2 v}{\partial x^2} + \frac{\partial^4 v}{\partial x^4} = 0, \quad x \in [0, L], \quad t > 0,
\end{equation}
where $v(x,t)$ is the scalar field, and $L = 64$ is the domain length. The KS equation exhibits spatiotemporal chaos and is a canonical model for pattern formation and turbulence in dissipative systems.

\paragraph{Initial and Boundary Conditions.}
Periodic boundary conditions are imposed: $v(0, t) = v(L, t)$ and $\partial^n v / \partial x^n |_{x=0} = \partial^n v / \partial x^n |_{x=L}$ for all $n$ and $t \geq 0$. The initial condition is constructed as a random superposition of Fourier modes:
\begin{equation}
v(x, 0) = \sum_{i=1}^{n_{\text{modes}}} A_i \sin\left(\frac{2\pi \ell_i x}{L} + \phi_i\right),
\end{equation}
where $n_{\text{modes}} = 10$, and for each mode $i$: $A_i \sim \mathcal{U}(-0.5, 0.5)$, $\phi_i \sim \mathcal{U}(0, 2\pi)$, and $\ell_i \sim \mathcal{U}_{\text{discrete}}\{1, 2, 3\}$ are randomly sampled once per trajectory. A warmup phase of $T_{\text{warmup}} = 50.0$ time units is applied before recording to allow the system to reach a statistically stationary chaotic state.

\paragraph{Spatial Discretization.}
The KS equation is discretized using a pseudo-spectral collocation method on a uniform grid of $N = 256$ points with grid spacing $\Delta x = L / N$. The solution is represented in Fourier space as $\hat{v}_k(t) = \mathcal{F}[v](k)$, where $k = (2\pi / L) \cdot k_{\text{index}}$ are the wavenumbers with $k_{\text{index}} \in \{-N/2, \ldots, N/2-1\}$. The spatial derivatives are computed exactly in Fourier space:
\begin{equation}
\frac{\partial^n v}{\partial x^n} \quad \Longleftrightarrow \quad (ik)^n \hat{v}_k.
\end{equation}
The nonlinear term $v \partial v / \partial x$ is evaluated in physical space and transformed back to Fourier space. To eliminate aliasing errors from the quadratic nonlinearity, the $2/3$-rule dealiasing is enforced: all Fourier modes with $|k_{\text{index}}| > N/3$ are set to zero both before computing the nonlinear term (on input) and after computing it (on output). This retains approximately $2N/3 + 1 \approx 171$ modes out of 256, ensuring that the product of two dealiased fields contains no aliasing.

\paragraph{Temporal Discretization.}
Time integration employs the fourth-order Exponential Time Differencing Runge-Kutta method (ETDRK4) with a fixed timestep $\Delta t = 0.01$. The KS equation is split into linear and nonlinear parts:
\begin{equation}
\frac{\partial \hat{v}_k}{\partial t} = \hat{L}_k \hat{v}_k + \hat{\mathcal{N}}_k(\hat{v}), \quad \hat{L}_k = k^2 - k^4, \quad \hat{\mathcal{N}}_k(\hat{v}) = -\widehat{v \frac{\partial v}{\partial x}}_k.
\end{equation}
The ETDRK4 scheme integrates the linear term exactly and treats the nonlinear term via four stages:
\begin{align}
\hat{v}^{(a)} &= e^{\Delta t \hat{L}_k / 2} \hat{v}^n + Q_k \hat{\mathcal{N}}_k(\hat{v}^n), \\
\hat{v}^{(b)} &= e^{\Delta t \hat{L}_k / 2} \hat{v}^n + Q_k \hat{\mathcal{N}}_k(\hat{v}^{(a)}), \\
\hat{v}^{(c)} &= e^{\Delta t \hat{L}_k / 2} \hat{v}^{(a)} + Q_k \left[2\hat{\mathcal{N}}_k(\hat{v}^{(b)}) - \hat{\mathcal{N}}_k(\hat{v}^n)\right], \\
\hat{v}^{n+1} &= e^{\Delta t \hat{L}_k} \hat{v}^n + f_{1,k} \hat{\mathcal{N}}_k(\hat{v}^n) + 2 f_{2,k} \left[\hat{\mathcal{N}}_k(\hat{v}^{(a)}) + \hat{\mathcal{N}}_k(\hat{v}^{(b)})\right] + f_{3,k} \hat{\mathcal{N}}_k(\hat{v}^{(c)}),
\end{align}
where the coefficients $Q_k, f_{1,k}, f_{2,k}, f_{3,k}$ are computed using the contour integral method of Kassam and Trefethen to avoid numerical cancellation errors near $\hat{L}_k \approx 0$:
\begin{align}
Q_k &= \Delta t \int_0^{1/2} e^{\tau \Delta t \hat{L}_k} d\tau = \Delta t \cdot \frac{e^{\Delta t \hat{L}_k / 2} - 1}{\Delta t \hat{L}_k}, \\
f_{1,k} &= \Delta t \int_0^1 e^{(1-\tau) \Delta t \hat{L}_k} \left[-4 - \tau \Delta t \hat{L}_k + e^{\tau \Delta t \hat{L}_k} (4 - 3\tau \Delta t \hat{L}_k + \tau^2 (\Delta t \hat{L}_k)^2)\right] \frac{d\tau}{(\tau \Delta t \hat{L}_k)^3}, \\
f_{2,k} &= \Delta t \int_0^1 e^{(1-\tau) \Delta t \hat{L}_k} \left[2 + \tau \Delta t \hat{L}_k + e^{\tau \Delta t \hat{L}_k} (-2 + \tau \Delta t \hat{L}_k)\right] \frac{d\tau}{(\tau \Delta t \hat{L}_k)^3}, \\
f_{3,k} &= \Delta t \int_0^1 e^{(1-\tau) \Delta t \hat{L}_k} \left[-4 - 3\tau \Delta t \hat{L}_k - \tau^2 (\Delta t \hat{L}_k)^2 + e^{\tau \Delta t \hat{L}_k} (4 - \tau \Delta t \hat{L}_k)\right] \frac{d\tau}{(\tau \Delta t \hat{L}_k)^3}.
\end{align}
These integrals are evaluated numerically via 32-point trapezoidal rule on a circular contour in the complex plane with radius $r = 1$, ensuring high accuracy even when $\hat{L}_k = 0$ (the DC mode).

\paragraph{Simulation Protocol.}
Each trajectory begins from a random initial condition and undergoes a warmup phase of $T_{\text{warmup}} = 50.0$ time units (5000 timesteps) to reach the chaotic attractor. After warmup, snapshots are recorded for a duration of $T_{\text{max}} = 10.0$ time units at intervals $\Delta t = 0.01$, yielding $n_{\text{snapshots}} = 1001$ snapshots (including the initial state at $t = 0$ of the recording phase). A total of 64 trajectories are generated with different random initial conditions. The fine-grid snapshots at resolution $N = 256$ are coarsened to multiple resolutions via spectral downsampling (truncation in Fourier space followed by inverse FFT). For instance, with resample factor $r = 4$, the snapshots are downsampled to $N = 64$, yielding the coarse-grid reference states $\bar{v}_{\Delta}$ used for closure modeling. The ETDRK4 integrator with fixed timestep $\Delta t = 0.01$ ensures fourth-order temporal accuracy while maintaining exponential stability for the stiff linear operator.

\subsubsection{Decaying Burgers Equation}

\paragraph{Governing Equation.}
We consider the one-dimensional viscous Burgers equation in conservative form without external forcing:
\begin{equation}
\frac{\partial u}{\partial t} + \frac{\partial}{\partial x}\left(\frac{u^2}{2}\right) = \nu \frac{\partial^2 u}{\partial x^2}, \quad x \in [0, L], \quad t > 0,
\end{equation}
where $u(x,t)$ is the velocity field, $\nu = 5 \times 10^{-4}$ is the kinematic viscosity, and $L = 2\pi$ is the domain length with periodic boundary conditions. This configuration exhibits free decay of turbulent kinetic energy, making it a canonical test case for subgrid-scale (SGS) closure modeling.

\paragraph{Initial and Boundary Conditions.}
Periodic boundary conditions are enforced: $u(0, t) = u(L, t)$ for all $t \geq 0$. The initial condition is generated from a broadband energy spectrum following Maulik and San (2018):
\begin{equation}
E(k) = A \, k^4 \exp\left(-\left(\frac{k}{k_0}\right)^2\right), \quad A = \frac{2k_0^{-5}}{3\sqrt{\pi}},
\end{equation}
where $k_0 = 10$ is the peak wavenumber. For each trajectory, the initial velocity field is constructed as:
\begin{equation}
u(x, 0) = \sum_{k=1}^{N/2-1} \sqrt{2 E(k)} \, N \, \cos\left(k \frac{2\pi x}{L} + \phi_k\right),
\end{equation}
where $\phi_k \sim \mathcal{U}(0, 2\pi)$ are independent random phases ensuring a real-valued field. This initialization produces a turbulent state with energy concentrated around intermediate wavenumbers, which subsequently undergoes viscous dissipation without external forcing.

\paragraph{Spatial Discretization.}
The equation is discretized using a pseudo-spectral collocation method on a uniform grid of $N = 2048$ points with grid spacing $\Delta x = L / N$. The solution is represented in Fourier space as $\hat{u}_k(t) = \mathcal{F}[u](k)$, where $k = (2\pi / L) \cdot k_{\text{index}}$ are the wavenumbers. Linear operators (diffusion) are computed exactly in Fourier space:
\begin{equation}
\frac{\partial^2 u}{\partial x^2} \quad \Longleftrightarrow \quad -k^2 \hat{u}_k.
\end{equation}
The nonlinear convective term $\partial(u^2/2)/\partial x$ is evaluated in physical space via the collocation method: $u^2/2$ is computed pointwise, transformed to Fourier space, then multiplied by $ik$. To eliminate aliasing errors from the quadratic nonlinearity, the $2/3$-rule dealiasing is applied: all Fourier modes with $|k_{\text{index}}| > N/3$ are set to zero both before computing the nonlinear term (on input) and after transforming back to Fourier space (on output). This retains approximately $2N/3 + 1 \approx 1366$ modes out of 2048.

\paragraph{Temporal Discretization.}
Time integration employs the third-order Total Variation Diminishing Runge-Kutta method (TVD-RK3) with adaptive timestep selection based on the CFL condition. The TVD-RK3 scheme proceeds via three stages:
\begin{align}
u^{(1)} &= u^n + \Delta t \, \mathcal{L}(u^n), \\
u^{(2)} &= \frac{3}{4} u^n + \frac{1}{4} \left[u^{(1)} + \Delta t \, \mathcal{L}(u^{(1)})\right], \\
u^{n+1} &= \frac{1}{3} u^n + \frac{2}{3} \left[u^{(2)} + \Delta t \, \mathcal{L}(u^{(2)})\right],
\end{align}
where $\mathcal{L}(u)$ denotes the spatial operator including both convective and diffusive terms. At each RK substep, the timestep $\Delta t$ is computed adaptively to satisfy the CFL condition with safety factor $\nu_{\text{CFL}} = 0.4$:
\begin{equation}
\Delta t = \nu_{\text{CFL}} \cdot \min\left(\frac{\Delta x}{\max_i |u_i|}, \frac{\Delta x^2}{2\nu}\right),
\end{equation}
where the first term ensures convective stability and the second term ensures diffusive stability. The timestep is recomputed at every substep to adapt to the instantaneous solution state.

\paragraph{Simulation Protocol.}
Each trajectory evolves from $t = 0$ to $T_{\text{final}} = 0.1$ time units. During this period, $n_{\text{snapshots}} = 1001$ snapshots (including the initial state) are saved at uniform intervals $\Delta t_{\text{outer}} = T_{\text{final}} / (n_{\text{snapshots}} - 1) = 10^{-4}$. A total of 64 trajectories are generated with different random initial phases. The fine-grid snapshots at resolution $N = 2048$ are coarsened to multiple resolutions using two methods: (i) spectral truncation, which retains the lowest $N_{\text{coarse}}$ Fourier modes and performs inverse FFT with renormalization, and (ii) direct subsampling, which selects every $r$-th grid point. For instance, with resample factor $r = 8$, spectral truncation yields $N = 256$ while preserving large-scale structures, whereas subsampling produces the same resolution but may introduce aliasing. The adaptive timestep strategy results in typical internal timesteps in the range $\Delta t \approx 10^{-5}$ to $10^{-4}$, requiring approximately $10^3$ to $10^4$ integration steps per trajectory depending on the instantaneous maximum velocity (typically $\max|u| \approx 2$ to $4$ during the decay phase).

\subsubsection{Forcing Navier--Stokes: Kolmogorov Flow}

\paragraph{Governing Equations.}
We consider the two-dimensional incompressible Navier--Stokes equations with external forcing on a periodic domain:
\begin{align}
\frac{\partial \mathbf{v}}{\partial t} + (\mathbf{v} \cdot \nabla)\mathbf{v} &= -\nabla p + \nu \nabla^2 \mathbf{v} + \mathbf{f}(x, y), \quad (x, y) \in [0, 2\pi] \times [0, 2\pi], \quad t > 0, \\
\nabla \cdot \mathbf{v} &= 0,
\end{align}
where $\mathbf{v} = (u, v)$ is the velocity field, $p$ is the kinematic pressure, $\nu$ is the kinematic viscosity, and $\mathbf{f}$ is the external forcing. We simulate two Reynolds number regimes: $\text{Re} = 1000$ with $\nu = 10^{-3}$ and $\text{Re} = 4000$ with $\nu = 2.5 \times 10^{-4}$, corresponding to moderate and high turbulence intensity respectively. Periodic boundary conditions are enforced in both directions: $\mathbf{v}(0, y, t) = \mathbf{v}(2\pi, y, t)$ and $\mathbf{v}(x, 0, t) = \mathbf{v}(x, 2\pi, t)$ for all $x, y, t$.

\paragraph{Initial and Forcing Conditions.}
The initial velocity field is generated as a filtered random field with maximum amplitude $\|\mathbf{v}\|_{\max} = 7.0$, ensuring a well-defined turbulent state from the onset. A warmup phase of $T_{\text{warmup}} = 40.0$ time units is applied to allow the flow to reach a statistically stationary turbulent state before recording snapshots. The external forcing consists of two components: (i) Kolmogorov forcing $\mathbf{f}_K = (F_0 \sin(k_f y), 0)$ with $F_0 = 1.0$ and $k_f = 4$, which drives energy injection at large scales, and (ii) linear damping $\mathbf{f}_L = \alpha \mathbf{v}$ with $\alpha = -0.1$, which removes energy at the largest scales to prevent unrealistic energy accumulation. The combined forcing $\mathbf{f} = \mathbf{f}_K + \mathbf{f}_L$ maintains a statistically steady turbulent cascade with energy injection at intermediate wavenumbers and dissipation at both large scales (via linear damping) and small scales (via viscosity).

\paragraph{Spatial Discretization.}
The equations are discretized on an Arakawa C-grid with resolution $N \times N = 512 \times 512$ cells and uniform grid spacing $\Delta x = \Delta y = 2\pi / 512$. Velocity components are staggered: the $u$-component is stored at cell faces with offset $(1, 0.5)$ and the $v$-component at faces with offset $(0.5, 1)$, while pressure is stored at cell centers with offset $(0.5, 0.5)$. The convective term $(\mathbf{v} \cdot \nabla)\mathbf{v}$ is computed using the van Leer flux-limiting scheme with TVD property, which applies second-order accurate upwind-biased reconstruction in smooth regions and reverts to first-order upwinding near sharp gradients to maintain monotonicity. The diffusive term $\nu \nabla^2 \mathbf{v}$ is computed using second-order centered finite differences. The incompressibility constraint $\nabla \cdot \mathbf{v} = 0$ is enforced via a fractional-step pressure projection method: after advancing the momentum equation, the velocity field is projected onto the divergence-free subspace by solving the Poisson equation $\nabla^2 \phi = \nabla \cdot \mathbf{v}^*$ for the pressure correction $\phi$, then updating $\mathbf{v}^{n+1} = \mathbf{v}^* - \nabla \phi$. The Poisson equation is solved using the fast diagonalization method, which exploits separability of the Laplacian operator to reduce the problem to a sum of one-dimensional eigenvalue problems, achieving $O(N^2 \log N)$ complexity for the 2D case.

\paragraph{Temporal Discretization.}
Time integration employs a first-order explicit-implicit splitting scheme: the convective and forcing terms are treated explicitly using forward Euler, while the diffusive term is treated implicitly to avoid the restrictive diffusive CFL constraint. The timestep is chosen adaptively based on the CFL condition for advection with safety factor $C_{\max} = 0.5$:
\begin{equation}
\Delta t = C_{\max} \cdot \frac{\Delta x}{\|\mathbf{v}\|_{\max}},
\end{equation}
where $\|\mathbf{v}\|_{\max} = 7.0$ is the maximum velocity magnitude observed in the initial condition. This yields $\Delta t = 0.5 \times (2\pi / 512) / 7.0 \approx 5.59 \times 10^{-4}$. The resulting CFL number $\text{CFL} = \|\mathbf{v}\|_{\max} \Delta t / \Delta x = 0.5$ is maintained throughout the simulation. Each time integration step consists of three substeps: (1) explicit momentum update $\mathbf{v}^* = \mathbf{v}^n + \Delta t [- (\mathbf{v}^n \cdot \nabla)\mathbf{v}^n + \mathbf{f}]$, (2) pressure projection to enforce incompressibility $\mathbf{v}^{**} = \mathcal{P}(\mathbf{v}^*)$, and (3) implicit diffusion solve $(I - \nu \Delta t \nabla^2)\mathbf{v}^{n+1} = \mathbf{v}^{**}$, where $\mathcal{P}$ denotes the projection operator onto divergence-free fields.

\paragraph{Simulation Protocol.}
Each trajectory begins from a random filtered initial condition and undergoes a warmup phase of $T_{\text{warmup}} = 40.0$ time units to reach statistically stationary turbulence. After warmup, snapshots are recorded for a production duration of $T_{\text{production}} \approx 0.0448$ time units. The production phase uses inner batching with $n_{\text{inner}} = 8$ timesteps per saved frame and $n_{\text{outer}} = 1000$ saved frames, yielding a total of $n_{\text{snapshots}} = 1001$ snapshots (including the initial state at $t = 0$ of the production phase) saved at uniform intervals $\Delta t_{\text{frame}} = 8 \times \Delta t \approx 4.47 \times 10^{-3}$. A total of 64 trajectories are generated with different random initial conditions. The fine-grid snapshots at resolution $512 \times 512$ are coarsened to multiple resolutions via the staggered velocity downsampling method, which preserves the divergence-free property: for each velocity component, only values lying on coarse-grid control volume faces are retained, then averaged over the appropriate face area. For instance, with resample factor $r = 16$, the snapshots are downsampled to $32 \times 32$, yielding the coarse-grid reference states $\bar{\mathbf{v}}_{\Delta}$ used for subgrid-scale closure modeling. This conservative downsampling ensures that coarse-grained fields satisfy $\nabla \cdot \bar{\mathbf{v}}_{\Delta} = 0$ exactly when the fine-grid field is divergence-free. The CFL-based adaptive timestep strategy maintains numerical stability throughout the simulation, with the timestep scaling inversely with grid resolution to keep the CFL number constant across different spatial resolutions.

\subsubsection{Decaying Navier--Stokes Turbulence}

\paragraph{Governing Equations.}
We consider the two-dimensional incompressible Navier--Stokes equations without external forcing on a periodic domain:
\begin{align}
\frac{\partial \mathbf{v}}{\partial t} + (\mathbf{v} \cdot \nabla)\mathbf{v} &= -\nabla p + \nu \nabla^2 \mathbf{v}, \quad (x, y) \in [0, 2\pi] \times [0, 2\pi], \quad t > 0, \\
\nabla \cdot \mathbf{v} &= 0,
\end{align}
where $\mathbf{v} = (u, v)$ is the velocity field, $p$ is the kinematic pressure, and $\nu = 10^{-3}$ is the kinematic viscosity, corresponding to an initial Reynolds number $\text{Re} = 1000$. The absence of external forcing results in free decay of kinetic energy, where turbulent structures evolve under the competing effects of nonlinear energy transfer and viscous dissipation. Periodic boundary conditions are enforced in both directions: $\mathbf{v}(0, y, t) = \mathbf{v}(2\pi, y, t)$ and $\mathbf{v}(x, 0, t) = \mathbf{v}(x, 2\pi, t)$ for all $x, y, t$. This configuration serves as a canonical test case for subgrid-scale (SGS) closure modeling in non-stationary flows, where the effective Reynolds number decreases over time as kinetic energy dissipates.

\paragraph{Initial and Forcing Conditions.}
The initial velocity field is generated as a filtered random field with maximum amplitude $\|\mathbf{v}\|_{\max} = 4.2$, producing a turbulent state with energy concentrated at intermediate wavenumbers. A warmup phase of $T_{\text{warmup}} = 4.5$ time units is applied to allow transient instabilities to decay and the flow to settle onto a smooth decay trajectory, ensuring that the recorded production phase represents well-resolved decaying turbulence rather than initialization artifacts. No external forcing is applied ($\mathbf{f} = \mathbf{0}$), so the total kinetic energy $E(t) = \frac{1}{2} \int |\mathbf{v}|^2 \, dx \, dy$ decreases monotonically over time due to viscous dissipation. This decay process is characterized by an inverse energy cascade at large scales and forward enstrophy cascade at small scales, typical of two-dimensional turbulence. The initial Reynolds number $\text{Re}_0 = 1000$ based on the initial maximum velocity and viscosity gradually decreases as the flow decays, providing a time-varying testbed for closure models.

\paragraph{Spatial Discretization.}
The spatial discretization is identical to the forced Kolmogorov flow case: an Arakawa C-grid with resolution $N \times N = 512 \times 512$ cells and uniform grid spacing $\Delta x = \Delta y = 2\pi / 512$. Velocity components are staggered with the $u$-component at offset $(1, 0.5)$ and the $v$-component at offset $(0.5, 1)$, while pressure resides at cell centers with offset $(0.5, 0.5)$. The convective term $(\mathbf{v} \cdot \nabla)\mathbf{v}$ is computed using the van Leer flux-limiting scheme with TVD property, ensuring second-order accuracy in smooth regions and monotonicity preservation near gradients. The diffusive term $\nu \nabla^2 \mathbf{v}$ is computed using second-order centered finite differences. Incompressibility $\nabla \cdot \mathbf{v} = 0$ is enforced via fractional-step pressure projection: the intermediate velocity $\mathbf{v}^*$ is projected onto the divergence-free subspace by solving the Poisson equation $\nabla^2 \phi = \nabla \cdot \mathbf{v}^*$ for the pressure correction $\phi$, then updating $\mathbf{v}^{n+1} = \mathbf{v}^* - \nabla \phi$. The Poisson solve employs the fast diagonalization method exploiting the separability of the Laplacian on periodic domains, achieving $O(N^2 \log N)$ complexity.

\paragraph{Temporal Discretization.}
Time integration employs the same first-order explicit-implicit splitting scheme as the forced case: convective terms are treated explicitly using forward Euler, while diffusive terms are treated implicitly to avoid the restrictive diffusive CFL constraint. The timestep is chosen based on the CFL condition for advection with safety factor $C_{\max} = 0.5$:
\begin{equation}
\Delta t = C_{\max} \cdot \frac{\Delta x}{\|\mathbf{v}\|_{\max}},
\end{equation}
where $\|\mathbf{v}\|_{\max} = 4.2$ is the initial maximum velocity magnitude. This yields $\Delta t \approx 1.46 \times 10^{-3}$, with an initial CFL number $\text{CFL}_0 = 0.5$. As the flow decays and $\|\mathbf{v}\|_{\max}$ decreases, the actual CFL number drops below the initial value, ensuring numerical stability throughout the simulation. Each time integration step consists of three substeps: (1) explicit momentum update $\mathbf{v}^* = \mathbf{v}^n + \Delta t [- (\mathbf{v}^n \cdot \nabla)\mathbf{v}^n]$, (2) pressure projection $\mathbf{v}^{**} = \mathcal{P}(\mathbf{v}^*)$ to enforce incompressibility, and (3) implicit diffusion solve $(I - \nu \Delta t \nabla^2)\mathbf{v}^{n+1} = \mathbf{v}^{**}$, where $\mathcal{P}$ denotes the projection operator onto divergence-free fields.

\paragraph{Simulation Protocol.}
Each trajectory begins from a random filtered initial condition and undergoes a warmup phase of $T_{\text{warmup}} = 4.5$ time units to allow transient dynamics to settle. After warmup, snapshots are recorded for a production duration of $T_{\text{production}} \approx 0.0117$ time units. The production phase uses inner batching with $n_{\text{inner}} = 8$ timesteps per saved frame and $n_{\text{outer}} = 1000$ saved frames, yielding a total of $n_{\text{snapshots}} = 1001$ snapshots (including the initial state at $t = 0$ of the production phase) saved at uniform intervals $\Delta t_{\text{frame}} = 8 \times \Delta t \approx 1.17 \times 10^{-2}$. A total of 64 trajectories are generated with different random initial conditions. The fine-grid snapshots at resolution $512 \times 512$ are coarsened to multiple resolutions via the same staggered velocity downsampling method used for forced flows, which preserves the divergence-free property: for each velocity component, only values on coarse-grid control volume faces are retained and averaged. For instance, with resample factor $r = 16$, the snapshots are downsampled to $32 \times 32$, yielding coarse-grid reference states $\bar{\mathbf{v}}_{\Delta}$ for subgrid-scale closure modeling. During the production phase, typical kinetic energy decay ranges from 5\% to 15\% depending on the trajectory, with the effective Reynolds number decreasing correspondingly. This decaying turbulence configuration provides a complementary testbed to the statistically stationary forced case, allowing assessment of closure model performance in time-evolving flows where the underlying turbulence statistics are non-stationary.

\subsection{Model Specifications}
\label{sec:model_specs}

\subsubsection{Models for 1D Benchmarks}

\paragraph{SINO for 1D Benchmarks.}
We employ a Scale-Invariant Neural Operator (SINO) architecture that operates on coarse-resolution inputs of shape $(N, C_{\text{in}})$ where $N$ is spatial resolution and $C_{\text{in}}$ is the number of input channels. The model consists of a lifting layer that projects inputs to hidden dimension $D=16$, followed by $L=2$ SINO blocks, and a projection layer that maps back to output channels. Each SINO block comprises dual branches: a frequency branch that applies spectral convolution via MLP-generated complex filters in Fourier space, and a spatial branch that performs circular convolution with MLP-generated kernels of size $k=9$. Both branches use coordinate-based MLPs with two hidden layers of width $W_{\text{freq}}=8$ and $W_{\text{spatial}}=8$ respectively, controlling the expressiveness of frequency and spatial representations. The penultimate layers of these MLPs have dimensions $N_{\text{freq}}=2$ and $N_{\text{spatial}}=2$. Features from both branches are concatenated and fused via a linear layer, with residual connections across blocks. Physical normalization is applied to both spatial coordinates and spectral frequencies to ensure scale invariance across different resolutions. The output layer uses small initialization (standard deviation $0.01$) to ensure training stability in early iterations. All activations use ReLU, and all spatial convolutions respect periodic boundary conditions through circular padding.

\paragraph{UNet for 1D Benchmarks.}
We employ a two-level U-Net architecture for learning subgrid-scale corrections in one-dimensional problems. The network operates on input feature maps of shape $(C_{\text{in}}, N)$ where $C_{\text{in}}$ denotes the number of input channels (typically 1 for velocity $u$) and $N$ is the spatial resolution. The encoder path downsamples the input twice via max pooling (factor of 2 each), progressively increasing channels from $C_{\text{in}}$ to $C_{\text{max}}/4$ to $C_{\text{max}}/2$, with each encoder stage consisting of two convolutions (kernel size $k=9$) followed by ReLU activations. The bottleneck layer operates at resolution $N/4$ with $C_{\text{max}} = 16$ channels. The decoder path mirrors the encoder with transposed convolutions for upsampling (stride 2, kernel size 2) and skip connections that concatenate encoder features before each decoder stage. All convolutions use circular padding to respect periodic boundary conditions. A final $1 \times 1$ convolution maps to a single output channel (the correction term $\tau$) with small initialization (standard deviation $0.01$) to ensure training stability in early iterations. The architecture follows the channel progression $1 \to C_{\text{max}}/4 \to C_{\text{max}}/2 \to C_{\text{max}} \to C_{\text{max}}/2 \to C_{\text{max}}/4 \to 1$ and spatial progression $N \to N/2 \to N/4 \to N/2 \to N$, requiring $N$ divisible by 4.

\paragraph{DeepONet for 1D Benchmarks.}
We employ a Deep Operator Network (DeepONet) that learns operator mappings through a decomposed architecture separating input functions from output coordinates. The model consists of two parallel subnetworks: a branch network that processes input function samples $\mathbf{u} \in \mathbb{R}^{m}$ at $m=32$ fixed sensor locations, and a trunk network that processes output coordinates $\mathbf{y}$ augmented with Fourier features. Both networks are MLPs with hidden dimension $D=32$, comprising $L_{\text{branch}}=4$ and $L_{\text{trunk}}=4$ layers respectively, using hyperbolic tangent activations. The trunk network input is enhanced with $n_{\text{Fourier}}=4$ frequency terms, forming a feature vector $[\mathbf{y}/L, 1, \cos(2\pi\mathbf{y}/L), \sin(2\pi\mathbf{y}/L), \ldots, \cos(8\pi\mathbf{y}/L), \sin(8\pi\mathbf{y}/L)]$ of dimension $1 + 1 + 2 \times 4 = 9$, where $L$ denotes the domain length. The branch network maps sensor readings to feature space $\mathbf{b} = \phi_{\text{branch}}(\mathbf{u}) \in \mathbb{R}^{p}$ with $p=32$ features, while the trunk network independently maps augmented coordinates to $\mathbf{t}(\mathbf{y}) = \phi_{\text{trunk}}(\tilde{\mathbf{y}}) \in \mathbb{R}^{p}$. The final output is computed as the inner product $G(\mathbf{u})(\mathbf{y}) = \langle \mathbf{b}, \mathbf{t}(\mathbf{y}) \rangle + b_0$, where $b_0$ is a learnable bias. The output layer uses small initialization (standard deviation $0.01$) to ensure training stability in early iterations. This architecture enables resolution-independent predictions and naturally supports periodic boundary conditions through Fourier coordinate encoding.

\paragraph{Transolver for 1DBenchmarks.}
We employ a Transformer-based operator network (Transolver) that incorporates physics-aware attention mechanisms for learning PDE solutions. The model operates on input feature maps of shape $(C_{\text{in}}, N)$ where $C_{\text{in}}=2$ includes the velocity field and spatial coordinates, and $N$ is the spatial resolution. The architecture begins with a lifting layer that maps inputs to hidden dimension $D=16$, followed by $L=2$ Transolver blocks. Each block implements a three-step physics attention mechanism: (1) \textit{Slice} adaptively groups $N$ grid points into $G=32$ learnable slice tokens via temperature-scaled softmax weights, aggregating local features into macroscopic patterns; (2) \textit{Self-Attention} applies multi-head attention ($H=8$ heads) among slice tokens to capture global interactions, where query, key, and value projections operate on head dimension $d_h = D/H = 2$; (3) \textit{Deslice} redistributes attended slice token information back to original grid points using the same grouping weights. Each block includes residual connections and a feed-forward network with expansion ratio $r=4$ (hidden dimension $4D=64$), using GELU activations and layer normalization. The final projection layers map from hidden dimension through an intermediate layer of 128 channels to the output dimension $C_{\text{out}}=1$. The output layer uses small initialization (standard deviation $0.01$) to ensure training stability in early iterations. This architecture enables adaptive spatial receptive fields through learned grouping while maintaining global awareness through attention mechanisms.

\paragraph{OFormer for 1D Benchmarks.}
We employ an Operator Transformer (OFormer) that learns operator mappings through a four-stage encoder-cross-decoder architecture with Galerkin-type linear attention. The model processes input functions $\mathbf{u}(\mathbf{x})$ sampled at $N$ points with feature dimension $C_{\text{in}}$ and outputs predictions at arbitrary query locations. The architecture begins with an encoder that lifts inputs to hidden dimension $D=16$ via a linear layer with GELU activation. The encoded features $\mathbf{z}_{\text{enc}}$ are then processed by a cross-attention module that extracts relevant information from input locations to query points using $H=8$ attention heads with head dimension $d_h = D/H = 2$. This cross-attention employs rotary position embeddings (RoPE) with minimum frequency $1/64$ to encode spatial relationships, and applies Galerkin-type normalization through instance normalization across spatial dimensions for both keys and values. The attended features $\mathbf{z}_{\text{cross}}$ are refined through $L=2$ propagator layers, each consisting of pre-normalized linear self-attention with residual connections followed by a feed-forward network with hidden dimension $2D=32$ and GELU activation. Linear attention achieves $\mathcal{O}(N)$ complexity by computing $\mathbf{Q}(\mathbf{K}^\top\mathbf{V})$ instead of $\text{softmax}(\mathbf{QK}^\top)\mathbf{V}$, with scaling factor $1/d_h$ for numerical stability. The decoder projects refined features from hidden dimension to output channels $C_{\text{out}}$ using a linear layer with small initialization (standard deviation $0.01$) to ensure training stability in early iterations. RoPE enables the model to generalize to arbitrary query point distributions while maintaining spatial awareness through frequency-based position encoding.

\paragraph{Galerkin Transformer for 1D Benchmarks.}
We employ a Galerkin Transformer that learns integral operators through linear attention mechanisms with $\mathcal{O}(N)$ computational complexity. The model processes input physical quantities $\mathbf{u}(\mathbf{x})$ sampled at $N$ spatial locations with feature dimension $C_{\text{in}}$. The architecture begins with a lifting layer that maps inputs pointwise to hidden dimension $D=16$ via a linear transformation. The lifted features are then refined through $L=2$ encoder layers, each implementing a Galerkin attention mechanism followed by a position-wise feed-forward network with residual connections and layer normalization. The Galerkin attention employs $H=8$ heads with head dimension $d_h = D/H = 2$, computing linear attention as $\text{Attention}(\mathbf{Q}, \mathbf{K}, \mathbf{V}) = \mathbf{Q}(\mathbf{K}^\top\mathbf{V})/N$, where the matrix product $\mathbf{K}^\top\mathbf{V} \in \mathbb{R}^{d_h \times d_h}$ aggregates global information independently of sequence length. This formulation avoids the quadratic complexity of softmax-based attention while maintaining global receptive fields across all spatial locations. Each encoder layer includes a feed-forward network with hidden dimension $D_{\text{FFN}}=512$ and ReLU activation. The refined features are decoded through a pointwise regressor consisting of two linear layers with ReLU activation, projecting from hidden dimension through an intermediate layer of dimension $D$ to output channels $C_{\text{out}}$. The output layer uses small initialization (standard deviation $0.01$) to ensure training stability in early iterations. This architecture enables efficient global information propagation while maintaining interpretability through its connection to Galerkin projection methods in numerical PDEs.

\paragraph{Fourier Neural Operator for 1D Benchmarks.}
We employ a Fourier Neural Operator (FNO) that learns integral operators in the frequency domain, providing resolution-invariant and globally receptive mappings. The model processes input functions $\mathbf{u}(\mathbf{x})$ augmented with normalized spatial coordinates, forming a feature vector with dimension $C_{\text{in}}=2$ at $N$ grid points. The architecture begins with a lifting layer that pointwise projects inputs to hidden dimension $D=16$ via a linear transformation. The lifted features are then refined through $L=2$ Fourier layers, each implementing a parallel combination of spectral and spatial branches. The spectral branch computes $\mathcal{K}(\mathbf{v}) = \mathcal{F}^{-1}(\mathbf{R}_\ell \cdot \mathcal{F}(\mathbf{v})_{|k < K})$, where $\mathcal{F}$ denotes the real-valued Fast Fourier Transform (FFT), $\mathbf{R}_\ell \in \mathbb{C}^{D \times D \times K}$ are learnable complex weights with $K=16$ preserved low-frequency modes, and high-frequency components beyond mode $K$ are implicitly truncated to zero, providing spectral regularization. This truncation restricts the learned operator to spatial scales larger than $L/K$, where $L$ is the domain length. The spatial branch applies a $1 \times 1$ convolution (pointwise linear transformation) $\mathbf{W}_\ell \in \mathbb{R}^{D \times D}$ to capture local residual corrections. Each Fourier layer computes $\mathbf{v}_{\ell+1} = \sigma(\mathbf{W}_\ell \mathbf{v}_\ell + \mathcal{K}_\ell(\mathbf{v}_\ell))$, where GELU activation $\sigma$ is applied to the summed output for all but the final layer. The refined features are decoded through a two-layer pointwise MLP with fixed intermediate dimension $128$: the output projection applies $D \to 128 \to C_{\text{out}}$ with GELU activation between layers and small initialization (standard deviation $0.01$) on the final layer for training stability. The frequency-domain formulation achieves global receptive fields covering the entire spatial domain in each layer, while the modes parameter $K$ controls the frequency bandwidth, balancing expressive power against overfitting to high-frequency noise.

\paragraph{U-Net Enhanced Fourier Neural Operator for 1D Benchmarks.}
We employ a U-Net enhanced Fourier Neural Operator (UFNO) that combines the global receptive fields of FNO with the local multiscale feature extraction capabilities of U-Net. The model processes input functions $\mathbf{u}(\mathbf{x})$ augmented with normalized spatial coordinates, forming feature vectors with dimension $C_{\text{in}}=2$ at $N$ grid points. The architecture begins with a lifting layer that pointwise projects inputs to hidden dimension $D=16$ via a linear transformation. The lifted features are refined through $L=2$ UFNO layers, each implementing a three-branch parallel architecture: $\mathbf{v}_{\ell+1} = \sigma(\mathcal{K}_\ell(\mathbf{v}_\ell) + \mathbf{W}_\ell(\mathbf{v}_\ell) + \mathcal{U}_\ell(\mathbf{v}_\ell))$. The spectral branch $\mathcal{K}_\ell$ computes frequency-domain convolutions $\mathcal{F}^{-1}(\mathbf{R}_\ell \cdot \mathcal{F}(\mathbf{v})_{|k < K})$ with $K=16$ preserved modes, capturing global dependencies through truncated Fourier transforms. The spatial branch $\mathbf{W}_\ell$ applies $1 \times 1$ convolutions for channel mixing and local residuals. The U-Net branch $\mathcal{U}_\ell$ implements a lightweight encoder-decoder architecture with one downsampling layer (factor 2), circular padding for periodic boundary compatibility, and skip connections that preserve multiscale spatial information from resolution $N$ to $N/4$, enabling fine-grained local feature refinement. We employ a staged activation strategy controlled by parameter $\ell_{\text{start}}=1$: layers $\ell < \ell_{\text{start}}$ use only spectral and spatial branches (FNO-only mode) for rapid coarse feature extraction, while layers $\ell \geq \ell_{\text{start}}$ incorporate the U-Net branch for detailed multiscale refinement, balancing computational efficiency with expressive power. GELU activation $\sigma$ is applied to all but the final layer. The refined features are decoded through a two-layer pointwise MLP projecting from hidden dimension through intermediate dimension $128$ to output channels $C_{\text{out}}$ with GELU activation and small initialization (standard deviation $0.01$). The U-Net refiner adds approximately $3D^2$ parameters per activated layer, while kernel size $k=9$ controls its local receptive field. This hybrid architecture combines FNO's frequency-domain global propagation with U-Net's spatial-domain multiscale analysis, enabling simultaneous capture of long-range dependencies and localized structures.

\paragraph{Amortized Fourier Neural Operator for 1D Benchmarks.}
We employ an Amortized Fourier Neural Operator (AMFNO) that dynamically generates frequency-domain convolution kernels through multi-layer perceptrons, enabling frequency-adaptive PDE solving. The model processes input functions $\mathbf{u}(\mathbf{x})$ at $N$ grid points, automatically augmenting them with normalized spatial coordinates to form feature vectors of dimension $C_{\text{in}}+1=2$. The architecture begins with a lifting layer that pointwise projects augmented inputs to hidden dimension $D=16$ via a linear transformation. The lifted features are refined through $L=2$ AMFNO layers, each implementing a dual-branch architecture with residual connections: $\mathbf{v}_{\ell+1} = \mathbf{v}_\ell + \mathcal{K}_{\text{MLP}}(\mathbf{v}_\ell) + \mathbf{W}_{\text{MLP}}(\mathbf{v}_\ell)$, where GELU activation $\sigma(\cdot)$ is applied to the summed output for all but the final layer. The spectral branch $\mathcal{K}_{\text{MLP}}$ replaces the fixed complex weights of standard FNO with dynamically generated kernels: for each frequency mode $\omega_k$ in the discrete Fourier spectrum, we encode the normalized frequency coordinate as a low-dimensional feature representation. Two separate MLPs with hidden dimension $4$ then map this frequency encoding to real and imaginary components of the convolution kernel: $K(\omega) = \text{MLP}_r(\omega) + i \cdot \text{MLP}_i(\omega) \in \mathbb{C}^{D \times D}$, producing mode-specific transformations across all $N/2+1$ frequencies in the real FFT spectrum. This dynamically generated kernel performs frequency-domain convolution $\mathcal{F}^{-1}(K(\omega) \cdot \mathcal{F}(\mathbf{v}))$, where $\mathcal{F}$ denotes the real-valued Fast Fourier Transform. The spatial branch $\mathbf{W}_{\text{MLP}}$ applies a two-layer pointwise MLP with hidden dimension $4D=64$ and GELU activation for local feature mixing. The refined features are decoded through a single linear projection from hidden dimension $D$ to output channels $C_{\text{out}}=1$. This frequency-adaptive kernel generation enables superior resolution generalization compared to fixed-kernel FNO, where each frequency mode receives a tailored transformation learned from data rather than using predetermined spectral truncation.

\subsubsection{Models for 2D Benchmarks}

\paragraph{SINO for 2D Benchmarks.}
We employ a Scale-Invariant Neural Operator (SINO) architecture that operates on coarse-resolution inputs of shape $(H, W, C_{\text{in}})$ where $H$ and $W$ are spatial resolutions and $C_{\text{in}}$ is the number of input channels. The model consists of a lifting layer that projects inputs to hidden dimension $D=32$, followed by $L=4$ SINO blocks, and a projection layer that maps back to output channels. Each SINO block comprises dual branches: a frequency branch that applies spectral convolution via MLP-generated complex filters in Fourier space using the 2D real FFT (rfft2), and a spatial branch that performs circular convolution with MLP-generated kernels of size $k=9 \times 9$. Both branches use coordinate-based MLPs with two hidden layers of width $w_{\text{freq}}=32$ and $w_{\text{spatial}}=32$ respectively, controlling the expressiveness of frequency and spatial representations. The penultimate layers of these MLPs have dimensions $n_{\text{freq}}=2$ and $n_{\text{spatial}}=2$, forming bottleneck layers that induce low-rank parameterization. For the frequency branch, physically normalized frequency coordinates $(\omega_y, \omega_x) \in [-1,1]^2$ are computed by mapping discrete wavenumbers to normalized ranges for each mode in the half-spectrum: $\xi_y = 4k_y/H - 1$ spanning the full Y-spectrum and $\xi_x = 4k_x/W - 1$ covering the non-negative X-half. The MLP generates mode-specific complex filters $K(\omega) = \text{MLP}_r(\omega) + i \cdot \text{MLP}_i(\omega) \in \mathbb{C}^{D \times D}$ across all $H \times (W/2+1)$ frequencies, with output normalization by $1/\sqrt{D}$ and percentile-based clipping for numerical stability. A bypass connection via a separate linear layer is added to the spectral output. For the spatial branch, normalized relative displacement coordinates $\zeta = (\zeta_y, \zeta_x) \in [-1, 1]^2$ are generated by mapping kernel offsets to normalized ranges independently in both dimensions, producing position-dependent convolution weights that respect periodic boundary conditions through circular padding. Features from both branches are concatenated and fused via a linear layer, with residual connections scaled by $0.1$ across blocks to prevent gradient explosion: $h^{(\ell+1)} = h^{(\ell)} + 0.1 \cdot W_{\text{fuse}}^{(\ell)} [h_{\text{freq}}^{(\ell)} \, ; \, h_{\text{spatial}}^{(\ell)}]$. Physical normalization is applied to both spatial coordinates and spectral frequencies to ensure scale invariance across different resolutions. The output layer uses small initialization (standard deviation $0.01$) to ensure training stability in early iterations. All activations use ReLU, and all spatial convolutions respect periodic boundary conditions through circular padding.

\paragraph{UNet for 2D Benchmarks.}
We employ a two-level U-Net architecture for learning subgrid-scale corrections in two-dimensional problems. The network operates on input feature maps of shape $(C_{\text{in}}, H, W)$ where $C_{\text{in}}$ denotes the number of input channels (typically 2 for velocity components $[u, v]$), $H$ is the height resolution, and $W$ is the width resolution. The encoder path downsamples the input twice via max pooling (factor of 2 each), progressively increasing channels from $C_{\text{in}}$ to $C_{\text{init}}$ to $C_{\text{init}} \times 2$, with each encoder stage consisting of two convolutions (kernel size $k=9$) followed by ReLU activations. The bottleneck layer operates at resolution $(H/4, W/4)$ with $C_{\text{max}} = 16$ channels. The decoder path mirrors the encoder with transposed convolutions for upsampling (stride 2, kernel size 2) and skip connections that concatenate encoder features before each decoder stage. All convolutions use circular padding to respect periodic boundary conditions. A final $1 \times 1$ convolution maps to output channels (the correction term $[\tau_u, \tau_v]$) with small initialization (standard deviation $0.01$) to ensure training stability in early iterations. The architecture follows the channel progression $2 \to C_{\text{init}} \to C_{\text{init}} \times 2 \to C_{\text{max}} \to C_{\text{init}} \times 2 \to C_{\text{init}} \to 2$ and spatial progression $(H, W) \to (H/2, W/2) \to (H/4, W/4) \to (H/2, W/2) \to (H, W)$, requiring both $H$ and $W$ divisible by 4.

\paragraph{DeepONet for 2D Benchmarks.}
We employ the same Deep Operator Network (DeepONet) architecture as in the 1D case, with input and output dimensions modified to accommodate two-dimensional velocity fields. The model processes input functions $[u, v]$ sampled on a fixed sensor grid of shape $(N_x, N_y) = (32, 32)$, flattened to dimension $m = 32 \times 32 \times 2 = 2048$ for the branch network. The trunk network processes two-dimensional output coordinates $(x, y)$ augmented with isotropic Fourier features of dimension $3 + 4 \times n_{\text{Fourier}} = 19$, where $n_{\text{Fourier}} = 4$. Both branch and trunk networks remain MLPs with hidden dimension $D=32$, $L_{\text{branch}}=4$ and $L_{\text{trunk}}=4$ layers, and hyperbolic tangent activations. The branch network maps sensor readings to feature space $\mathbf{b} = \phi_{\text{branch}}(\mathbf{u}) \in \mathbb{R}^{p}$ with $p=32$ features, while the trunk network maps augmented coordinates to $\mathbf{t}(x, y) = \phi_{\text{trunk}}(\tilde{x}, \tilde{y}) \in \mathbb{R}^{p}$. The final output is computed as $G(\mathbf{u})(x, y) = \langle \mathbf{b}, \mathbf{t}(x, y) \rangle + b_0$ with small initialization (standard deviation $0.01$) on the output layer. The 2D Fourier encoding naturally supports periodic boundary conditions in both spatial directions while maintaining resolution-independent predictions.

\paragraph{Transolver for 2D Benchmarks.}
We employ the same Transformer-based operator network (Transolver) architecture as in the 1D case, with input feature maps of shape $(C_{\text{in}}, H, W)$ where $C_{\text{in}}=4$ includes the two-dimensional velocity field $[u, v]$ and normalized spatial coordinates $[x/L_x, y/L_y]$. The architecture begins with a lifting layer that maps inputs to hidden dimension $D=32$, followed by $L=4$ Transolver blocks. Each block implements the same three-step physics attention mechanism adapted for 2D grids: (1) \textit{Slice} adaptively groups $N=H \times W$ grid points into $G=32$ learnable slice tokens via temperature-scaled softmax weights ($\tau=0.5$), aggregating local features into macroscopic patterns; (2) \textit{Self-Attention} applies multi-head attention ($H=8$ heads) among slice tokens with head dimension $d_h = D/H = 4$; (3) \textit{Deslice} redistributes attended information back to original grid points using the same grouping weights. We use linear projections for pointwise feature transformations, ensuring applicability to arbitrary geometries while maintaining parameter efficiency. Each block includes residual connections and a feed-forward network with expansion ratio $r=4$ (hidden dimension $4D=128$), using GELU activations and layer normalization. The final projection layers map from hidden dimension through an intermediate layer of 128 channels to output dimension $C_{\text{out}}=2$, with small initialization (standard deviation $0.01$) on the output layer. This architecture enables adaptive spatial receptive fields through learned grouping while maintaining global awareness through attention mechanisms.

\paragraph{OFormer for 2D Benchmarks.}
We employ the same Operator Transformer (OFormer) architecture as in the 1D case, implementing a four-stage encoder-cross-decoder framework with Galerkin-type linear attention for two-dimensional spatial domains. The model processes input functions $\mathbf{u}(\mathbf{x})$ sampled at $N=H \times W$ points with feature dimension $C_{\text{in}}=4$ (including velocity field $[u, v]$ and normalized coordinates $[x/L_x, y/L_y]$) and outputs predictions at arbitrary query locations in 2D space. The encoder lifts inputs to hidden dimension $D=32$ via a linear layer with GELU activation. The encoded features $\mathbf{z}_{\text{enc}}$ are then processed by a cross-attention module that extracts relevant information from input locations to query points using $H=8$ attention heads with head dimension $d_h = D/H = 4$. This cross-attention employs two-dimensional rotary position embeddings (RoPE) that apply separate frequency-based encodings to $x$ and $y$ coordinates with scale factor $16.0$ and minimum frequency $1/64$, naturally extending the 1D rotation mechanism to capture spatial relationships in the plane. Galerkin-type normalization is applied through instance normalization across spatial dimensions for both keys and values. The attended features $\mathbf{z}_{\text{cross}}$ are refined through $L=4$ propagator layers, each consisting of pre-normalized linear self-attention with residual connections followed by a feed-forward network with hidden dimension $2D=64$ and GELU activation. Linear attention achieves $\mathcal{O}(N)$ complexity by computing $\mathbf{Q}(\mathbf{K}^\top\mathbf{V})$ instead of $\text{softmax}(\mathbf{QK}^\top)\mathbf{V}$, with scaling factor $1/d_h$ for numerical stability. The decoder projects refined features from hidden dimension to output channels $C_{\text{out}}=2$ using a linear layer with small initialization (standard deviation $0.01$). The 2D RoPE extension enables the model to generalize to arbitrary query point distributions in the plane while maintaining spatial awareness through frequency-based position encoding in both spatial directions.

\paragraph{Galerkin Transformer for 2D Benchmarks.}
We employ the same Galerkin Transformer architecture as in the 1D case, learning integral operators through linear attention mechanisms with $\mathcal{O}(N)$ computational complexity for two-dimensional spatial domains. The model processes input physical quantities $\mathbf{u}(\mathbf{x})$ sampled at $N=H \times W$ spatial locations with feature dimension $C_{\text{in}}=4$ (including velocity field $[u, v]$ and normalized coordinates $[x/L_x, y/L_y]$). The 2D grid is flattened into a sequence of length $N$ for processing. The architecture begins with a lifting layer that maps inputs pointwise to hidden dimension $D=32$ via a linear transformation. The lifted features are then refined through $L=4$ encoder layers, each implementing a Galerkin attention mechanism followed by a position-wise feed-forward network with residual connections and layer normalization. The Galerkin attention employs $H=8$ heads with head dimension $d_h = D/H = 4$, computing linear attention as $\text{Attention}(\mathbf{Q}, \mathbf{K}, \mathbf{V}) = \mathbf{Q}(\mathbf{K}^\top\mathbf{V})/N$, where the matrix product $\mathbf{K}^\top\mathbf{V} \in \mathbb{R}^{d_h \times d_h}$ aggregates global information independently of sequence length. This formulation avoids the quadratic complexity of softmax-based attention while maintaining global receptive fields across all spatial locations in the 2D domain. Each encoder layer includes a feed-forward network with hidden dimension $D_{\text{FFN}}=512$ and ReLU activation. The refined features are decoded through a pointwise regressor consisting of two linear layers with ReLU activation, projecting from hidden dimension through an intermediate layer of dimension $D$ to output channels $C_{\text{out}}=2$. The output layer uses small initialization (standard deviation $0.01$) to ensure training stability in early iterations. This architecture enables efficient global information propagation across the entire 2D spatial domain while maintaining interpretability through its connection to Galerkin projection methods in numerical PDEs.

\paragraph{Fourier Neural Operator for 2D Benchmarks.}
We employ a Fourier Neural Operator (FNO) that learns integral operators in the frequency domain, providing resolution-invariant and globally receptive mappings for two-dimensional spatial domains. The model processes input functions $\mathbf{u}(\mathbf{x})$ augmented with normalized spatial coordinates, forming a feature vector with dimension $C_{\text{in}}=4$ (including velocity field $[u, v]$ and normalized coordinates $[x/L_x, y/L_y]$) at $N=H \times W$ grid points. The architecture begins with a lifting layer that pointwise projects inputs to hidden dimension $D=32$ via a linear transformation. The lifted features are then refined through $L=4$ Fourier layers, each implementing a parallel combination of spectral and spatial branches. The spectral branch computes $\mathcal{K}(\mathbf{v}) = \mathcal{F}^{-1}(\mathbf{R}_\ell^{(1)} \cdot \mathcal{F}(\mathbf{v})_{0:K_x,:K_y} + \mathbf{R}_\ell^{(2)} \cdot \mathcal{F}(\mathbf{v})_{-K_x:,:K_y})$, where $\mathcal{F}$ denotes the two-dimensional real-valued Fast Fourier Transform (rfft2), $\mathbf{R}_\ell^{(1)}, \mathbf{R}_\ell^{(2)} \in \mathbb{C}^{D \times D \times K_x \times K_y}$ are learnable complex weights with $(K_x, K_y)=(16, 16)$ preserved low-frequency modes in each spatial dimension, and high-frequency components beyond modes $(K_x, K_y)$ are implicitly truncated to zero, providing spectral regularization. The use of two separate weight matrices $\mathbf{R}_\ell^{(1)}$ and $\mathbf{R}_\ell^{(2)}$ for positive and negative frequencies in the $x$-direction enables the model to capture asymmetric spatial patterns while respecting the conjugate symmetry of real-valued signals. This truncation restricts the learned operator to spatial scales larger than $L_x/K_x$ and $L_y/K_y$ in the respective directions. The spatial branch applies a $1 \times 1$ convolution (pointwise linear transformation) $\mathbf{W}_\ell \in \mathbb{R}^{D \times D}$ to capture local residual corrections. Each Fourier layer computes $\mathbf{v}_{\ell+1} = \sigma(\mathbf{W}_\ell \mathbf{v}_\ell + \mathcal{K}_\ell(\mathbf{v}_\ell))$, where GELU activation $\sigma$ is applied to the summed output for all but the final layer. The refined features are decoded through a two-layer pointwise MLP with fixed intermediate dimension $128$: the output projection applies $D \to 128 \to C_{\text{out}}$ with GELU activation between layers and small initialization (standard deviation $0.01$) on the final layer for training stability. The frequency-domain formulation achieves global receptive fields covering the entire two-dimensional spatial domain in each layer, while the modes parameters $(K_x, K_y)$ control the frequency bandwidth in each direction, balancing expressive power against overfitting to high-frequency noise and enabling anisotropic frequency resolution for problems with directional characteristics.

\paragraph{U-Net Enhanced Fourier Neural Operator for 2D Benchmarks.}
We employ a U-Net enhanced Fourier Neural Operator (UFNO) that combines the global receptive fields of FNO with the local multiscale feature extraction capabilities of U-Net for two-dimensional spatial domains. The model processes input functions $\mathbf{u}(\mathbf{x})$ augmented with normalized spatial coordinates, forming feature vectors with dimension $C_{\text{in}}=4$ (including velocity field $[u, v]$ and normalized coordinates $[x/L_x, y/L_y]$) at $N=H \times W$ grid points. The architecture begins with a lifting layer that pointwise projects inputs to hidden dimension $D=32$ via a linear transformation. The lifted features are refined through $L=4$ UFNO layers, each implementing a three-branch parallel architecture: $\mathbf{v}_{\ell+1} = \sigma(\mathcal{K}_\ell(\mathbf{v}_\ell) + \mathbf{W}_\ell(\mathbf{v}_\ell) + \mathcal{U}_\ell(\mathbf{v}_\ell))$. The spectral branch $\mathcal{K}_\ell$ computes two-dimensional frequency-domain convolutions $\mathcal{F}^{-1}(\mathbf{R}_\ell^{(1)} \cdot \mathcal{F}(\mathbf{v})_{0:K_x,:K_y} + \mathbf{R}_\ell^{(2)} \cdot \mathcal{F}(\mathbf{v})_{-K_x:,:K_y})$ with $(K_x, K_y)=(16, 16)$ preserved modes in each spatial dimension, capturing global dependencies through truncated Fourier transforms with separate weight matrices for positive and negative frequencies. The spatial branch $\mathbf{W}_\ell$ applies $1 \times 1$ convolutions for channel mixing and local residuals. The U-Net branch $\mathcal{U}_\ell$ implements a lightweight encoder-decoder architecture with three downsampling layers (factors $2, 2, 2$) reaching resolution $H/8 \times W/8$, circular padding for periodic boundary compatibility, and skip connections that preserve multiscale spatial information across resolutions $H \times W$, $H/2 \times W/2$, $H/4 \times W/4$, and $H/8 \times W/8$, enabling fine-grained local feature refinement in both spatial directions. We employ a staged activation strategy controlled by parameter $\ell_{\text{start}}=2$: layers $\ell < \ell_{\text{start}}$ use only spectral and spatial branches (FNO-only mode) for rapid coarse feature extraction, while layers $\ell \geq \ell_{\text{start}}$ incorporate the U-Net branch for detailed multiscale refinement, balancing computational efficiency with expressive power. GELU activation $\sigma$ is applied to all but the final layer. The refined features are decoded through a single $1 \times 1$ convolution projecting from hidden dimension to output channels $C_{\text{out}}=2$ with small initialization (standard deviation $0.01$). The U-Net refiner adds approximately $3D^2$ parameters per activated layer through its convolutional structure, while kernel size $k=3$ controls its local receptive field in both spatial dimensions. This hybrid architecture combines FNO's frequency-domain global propagation with U-Net's spatial-domain multiscale analysis, enabling simultaneous capture of long-range dependencies across the entire 2D domain and localized structures at multiple spatial scales.

\paragraph{Amortized Fourier Neural Operator for 2D Benchmarks.}
We employ an Amortized Fourier Neural Operator (AMFNO) that dynamically generates frequency-domain convolution kernels through multi-layer perceptrons, enabling frequency-adaptive PDE solving for two-dimensional spatial domains. The model processes input functions $\mathbf{u}(\mathbf{x})$ at $N=H \times W$ grid points, automatically augmenting them with normalized spatial coordinates to form feature vectors of dimension $C_{\text{in}}+2=4$ (including velocity field $[u, v]$ and normalized coordinates $[x/L_x, y/L_y]$). The architecture begins with a lifting layer that pointwise projects augmented inputs to hidden dimension $D=32$ via a linear transformation. The lifted features are refined through $L=4$ AMFNO layers, each implementing a dual-branch architecture with residual connections: $\mathbf{v}_{\ell+1} = \mathbf{v}_\ell + \mathcal{K}_{\text{MLP}}(\mathbf{v}_\ell) + \mathbf{W}_{\text{MLP}}(\mathbf{v}_\ell)$, where GELU activation $\sigma(\cdot)$ is applied to the summed output for all but the final layer. The spectral branch $\mathcal{K}_{\text{MLP}}$ replaces the fixed complex weights of standard FNO with dynamically generated kernels through a separable two-dimensional frequency representation: for each frequency mode $(\omega_{k_x}, \omega_{k_y})$ in the discrete two-dimensional Fourier spectrum, we encode the normalized frequency coordinates independently using Chebyshev polynomial bases as low-dimensional feature representations. Four separate MLPs with hidden dimension $4$ then map these frequency encodings to real and imaginary components of separable convolution kernels: $K_x(\omega_{k_x}) = \text{MLP}_{xr}(\omega_{k_x}) + i \cdot \text{MLP}_{xi}(\omega_{k_x})$ and $K_y(\omega_{k_y}) = \text{MLP}_{yr}(\omega_{k_y}) + i \cdot \text{MLP}_{yi}(\omega_{k_y})$, with the final kernel formed through elementwise multiplication $K(\omega_{k_x}, \omega_{k_y}) = K_x(\omega_{k_x}) \odot K_y(\omega_{k_y}) \in \mathbb{C}^{D \times D}$, producing mode-specific transformations across all $H \times (W/2+1)$ frequencies in the real two-dimensional FFT spectrum. This dynamically generated kernel performs frequency-domain convolution $\mathcal{F}^{-1}(K(\omega_{k_x}, \omega_{k_y}) \cdot \mathcal{F}(\mathbf{v}))$, where $\mathcal{F}$ denotes the two-dimensional real-valued Fast Fourier Transform (rfft2). The spatial branch $\mathbf{W}_{\text{MLP}}$ applies a two-layer pointwise MLP with hidden dimension $4D=128$ and GELU activation for local feature mixing across both spatial dimensions. The refined features are decoded through a two-layer pointwise MLP projecting from hidden dimension through intermediate dimension $4D$ to output channels $C_{\text{out}}=2$ with GELU activation. This frequency-adaptive kernel generation with separable low-rank factorization enables superior resolution generalization compared to fixed-kernel FNO, where each frequency mode in both spatial directions receives a tailored transformation learned from data rather than using predetermined spectral truncation.

\subsection{Training Methodology}
\label{sec:training_methodology}

\paragraph{Neural Network Input and Output.}
Following the Indirect Neural Corrector (INC) framework~\citep{wei2026inc}, the neural network learns a closure term $\tau_\Delta$ that serves as a right-hand side correction to the coarse-grid governing equations. Given the current coarse-grid state $\bar{u}_\Delta^n$ at time step $n$, the network predicts the correction term $\tau_\Delta^n = \text{SINO}(\bar{u}_\Delta^n; \Theta)$, which is then incorporated into the time integration scheme. The corrected evolution follows:
\begin{equation}
	\frac{\partial \bar{u}_\Delta}{\partial t} = \mathcal{L}_\Delta(\bar{u}_\Delta) + \tau_\Delta,
\end{equation}
where $\mathcal{L}_\Delta$ denotes the coarse-grid spatial operator (e.g., WENO5 reconstruction for Burgers, pseudo-spectral derivatives for KS, van Leer flux limiting for NS). During rollout training, the predicted correction $\tau_\Delta^n$ is held constant across all substeps within each multi-stage time integrator (e.g., SSP-RK3, ETDRK4), enabling efficient gradient backpropagation while maintaining numerical stability. This formulation decouples the neural operator from the specific discretization scheme, allowing SINO to generalize across different numerical solvers.

\paragraph{Training Objective.}
We train the neural network correction term using a rollout loss with curriculum learning:
\begin{equation}
\mathcal{L}_K(\theta) = \frac{1}{B} \sum_{i=1}^B \frac{1}{K+1} \sum_{k=0}^K \| u_i^k - \hat{u}_i^k(\theta) \|^2
\end{equation}
where $B$ is the batch size, $K$ is the rollout length (curriculum parameter), $u_i^k$ is the ground truth at step $k$ for trajectory $i$, and $\hat{u}_i^k(\theta)$ is the model prediction.

\paragraph{Curriculum Learning.}
The rollout length $K$ grows progressively during training:
\begin{equation}
K(\text{epoch}) = \min\left(K_{\text{start}} + \left\lfloor \frac{\text{epoch}-1}{\Delta_{\text{epoch}}} \right\rfloor \cdot \Delta_K, K_{\text{max}}\right)
\end{equation}
where $K_{\text{start}}$ is the initial rollout length (default: 10 steps), $\Delta_{\text{epoch}}$ is the growth interval (default: every 10 epochs), $\Delta_K$ is the growth increment (default: 15 steps), and $K_{\text{max}}$ is the maximum rollout length corresponding to the training time horizon. This curriculum strategy stabilizes early training by first learning short-term dynamics, then gradually extending to longer time horizons.

\paragraph{Optimization.}
We use the AdamW optimizer with learning rate $\eta = 10^{-3}$ and cosine annealing schedule:
\begin{equation}
\eta(t) = \eta_0 \cdot 0.5 \left(1 + \cos\left(\frac{\pi t}{T}\right)\right)
\end{equation}
where $t$ is the current optimization step and $T$ is the total number of steps. For decaying Burgers turbulence, we apply weight decay $\lambda = 10^{-3}$ for regularization.

\paragraph{Data Normalization.}
For 1D equations (Burgers, KS), we normalize inputs using training set statistics: $u_{\text{norm}} = (u - \mu_{\text{train}}) / \sigma_{\text{train}}$ or $v_{\text{norm}} = (v - \mu_{\text{train}}) / \sigma_{\text{train}}$, where $\mu_{\text{train}}$ and $\sigma_{\text{train}}$ are computed exclusively from the training temporal window to prevent data leakage. The correction term is scaled back as $\tau = \tau_{\text{norm}} \cdot \sigma_{\text{train}}$. For 2D Navier-Stokes equations, we use standard deviation normalization: $\mathbf{v}_{\text{norm}} = \mathbf{v} / \sigma_{\text{train}}$.

\paragraph{Time Integration with Correction Term.}
The neural network learns a right-hand side correction term $\tau$ that is incorporated into the physical solver during rollout training. For time integration, the same correction term computed at the beginning of each time step is reused across all substeps of the multi-stage integrator. This approach reduces computational cost and provides clearer gradient paths during backpropagation through time.

\paragraph{Model Selection.}
We select the best model based on extrapolation performance, measured by the mean squared error on the test temporal window $[T_{\text{train}}, T_{\text{final}}]$.

\paragraph{Training Schedule.}
We train for 300 epochs with batch size 8 trajectories. All tasks use 1001 temporal snapshots: the first frame is the initial condition, the first 30\% (300 frames) form the training interval, and the remaining 70\% (700 frames) are reserved for extrapolation testing. All reported MSE errors are extrapolation test errors.

\paragraph{Time Stepping Configuration.}
We use different time stepping strategies tailored to each benchmark:

\textit{Forcing Burgers:} During the recording phase of $T_{\text{max}} = 10.0$ time units, $n_{\text{snapshots}} = 1001$ snapshots are saved at uniform intervals $\Delta t_{\text{outer}} = 0.01$. The DNS solver employs adaptive timesteps satisfying $\Delta t = \nu_{\text{CFL}} \cdot \min(\Delta x / \max_i |\bar{u}_i|, \Delta x^2 / (2\eta))$ with safety factor $\nu_{\text{CFL}} = 0.4$ and viscosity $\eta = 0.01$.

\textit{Decaying Burgers:} Snapshots are saved for $T_{\text{final}} = 0.1$ time units at intervals $\Delta t_{\text{outer}} = 10^{-4}$, yielding $n_{\text{snapshots}} = 1001$ snapshots. The DNS uses adaptive timesteps $\Delta t = \nu_{\text{CFL}} \cdot \min(\Delta x / \max_i |u_i|, \Delta x^2 / (2\nu))$ with $\nu_{\text{CFL}} = 0.4$ and viscosity $\nu = 5 \times 10^{-4}$.

\textit{Kuramoto-Sivashinsky:} After warmup of $T_{\text{warmup}} = 50.0$ time units, snapshots are recorded for $T_{\text{max}} = 10.0$ time units at intervals $\Delta t = 0.01$ using ETDRK4 integration, yielding $n_{\text{snapshots}} = 1001$ snapshots.

\textit{Forcing Navier-Stokes (Re=1000):} After warmup of $T_{\text{warmup}} = 40.0$ time units, snapshots are recorded for production duration $T_{\text{production}} \approx 0.0448$ time units. The DNS employs inner batching with $n_{\text{inner}} = 8$ timesteps per saved frame and $n_{\text{outer}} = 1000$ saved frames, yielding $n_{\text{snapshots}} = 1001$ snapshots at intervals $\Delta t_{\text{frame}} = 8 \times \Delta t$, where $\Delta t = C_{\max} \cdot \Delta x / \|\mathbf{v}\|_{\max}$ with $C_{\max} = 0.5$ and $\|\mathbf{v}\|_{\max} = 7.0$.

\textit{Forcing Navier-Stokes (Re=4000):} Identical configuration to Re=1000 case, but with viscosity $\nu = 2.5 \times 10^{-4}$ and the same maximum velocity $\|\mathbf{v}\|_{\max} = 7.0$.

\textit{Decaying Navier-Stokes:} After warmup of $T_{\text{warmup}} = 4.5$ time units, snapshots are recorded for production duration $T_{\text{production}} \approx 0.0117$ time units. The DNS uses $n_{\text{inner}} = 8$ timesteps per frame and $n_{\text{outer}} = 1000$ frames, yielding $n_{\text{snapshots}} = 1001$ snapshots at intervals $\Delta t_{\text{frame}} = 8 \times \Delta t$, where $\Delta t = C_{\max} \cdot \Delta x / \|\mathbf{v}\|_{\max}$ with $C_{\max} = 0.5$ and initial maximum velocity $\|\mathbf{v}\|_{\max} = 4.2$.

\paragraph{Random Seeds and Reproducibility.}
All 1D and 2D PDEs use seeds 0--63 to generate 64 trajectories via the JAX random number generator \texttt{jax.random.PRNGKey(seed)} for initializing velocity fields. Two trajectories were excluded due to numerical instabilities caused by extreme initial values: Decaying NS trajectory 61 (seed=61) and Taylor-Green NS trajectory 61 (seed=61). All other benchmarks use the complete set of 64 trajectories. For neural network training, each model configuration uses three independent random seeds for weight initialization (\texttt{model\_seed} $\in \{1, 2, 3\}$) to assess variance across random initializations, while the data loading seed (\texttt{seed}) is fixed at 42 for all experiments to ensure identical train-test splits. This design separates data randomness from model randomness, enabling fair comparison across architectures.

\paragraph{Computational Environment.}
All experiments were conducted on a workstation with Ubuntu 22.04, equipped with a single NVIDIA RTX 5090 GPU (32GB VRAM), 25 vCPU Intel Xeon Gold 6459C processor, Python 3.12, PyTorch 2.12.1, and CUDA 13.0. The JAX-based PDE solvers and neural network training utilized JIT compilation and automatic differentiation for computational efficiency.

\subsection{Theoretical Foundations: Low-Rank Structure and Lipschitz Regularity}
\label{sec:theory}

We provide mathematical analysis demonstrating that SINO's implicit kernel representation via bottleneck hypernetworks induces two fundamental properties: (1) explicit low-rank parameterization that aligns with the physical structure of closure operators, and (2) Lipschitz continuity that ensures smooth kernel variations and prevents overfitting to high-frequency noise. These properties jointly explain SINO's superior parameter efficiency and data efficiency in turbulent closure modeling. Our analysis establishes the low-rank structure of the learned kernels and quantitative bounds for operator norms, providing theoretical foundations for the empirical performance demonstrated in Section 4.

\subsubsection{Bottleneck-Induced Subspace Constraint}

\textbf{Architecture.} The frequency hypernetwork $\Phi_\omega: \mathbb{R}^2 \to \mathbb{C}^{D \times D}$ implements a three-layer bottleneck architecture:
\begin{equation}
	\xi \xrightarrow{W_1^{(f)}} \mathbb{R}^{w_f} \xrightarrow{\sigma} \mathbb{R}^{w_f} \xrightarrow{W_2^{(f)}} \mathbb{R}^{w_f} \xrightarrow{\sigma} \mathbb{R}^{w_f} \xrightarrow{W_3^{(f)}} \mathbb{R}^{n_f} \xrightarrow{\sigma} \mathbb{R}^{n_f} \xrightarrow{[W_r, W_i]} \mathbb{C}^{D \times D},
\end{equation}
where $\xi \in \mathbb{R}^2$ denotes normalized frequency coordinates, $W_1^{(f)} \in \mathbb{R}^{w_f \times 2}$, $W_2^{(f)} \in \mathbb{R}^{w_f \times w_f}$, $W_3^{(f)} \in \mathbb{R}^{n_f \times w_f}$ is the bottleneck layer with $n_f \ll w_f$, $W_r, W_i \in \mathbb{R}^{D^2 \times n_f}$ are output projection matrices (each column representing a vectorized $D \times D$ matrix), and $\sigma = \max(0, \cdot)$ is the ReLU activation function. The spatial hypernetwork $\Phi_x: \mathbb{R}^2 \to \mathbb{R}^{D \times D}$ follows an identical architecture with parameters $(W_1^{(s)} \in \mathbb{R}^{w_s \times 2}, W_2^{(s)} \in \mathbb{R}^{w_s \times w_s}, W_3^{(s)} \in \mathbb{R}^{n_s \times w_s}, W_s \in \mathbb{R}^{D^2 \times n_s})$ and bottleneck dimension $n_s \ll w_s$.

\begin{theorem}[Explicit Rank Constraint]
	\label{thm:explicit_rank_constraint}
	Let $z_\omega(\xi) = \sigma(W_3^{(f)} \sigma(W_2^{(f)} \sigma(W_1^{(f)} \xi))) \in \mathbb{R}^{n_f}$ denote the frequency bottleneck representation and $z_x(\zeta) = \sigma(W_3^{(s)} \sigma(W_2^{(s)} \sigma(W_1^{(s)} \zeta))) \in \mathbb{R}^{n_s}$ denote the spatial bottleneck representation. Then:
	
	\mbox{}
	
	\textbf{(a) Frequency domain:} The spectral kernel admits the explicit low-rank decomposition
	\begin{equation}
		m(\xi) = \sum_{j=1}^{n_f} z_{\omega,j}(\xi) \cdot (B_j^{(r)} + i B_j^{(i)}), \quad B_j^{(r)}, B_j^{(i)} \in \mathbb{R}^{D \times D},
	\end{equation}
	where $B_j^{(r)} = \mathrm{mat}_{D \times D}([W_r]_{\cdot,j})$ and $B_j^{(i)} = \mathrm{mat}_{D \times D}([W_i]_{\cdot,j})$ are obtained by reshaping the $j$-th column vectors $[W_r]_{\cdot,j}, [W_i]_{\cdot,j} \in \mathbb{R}^{D^2}$ into $D \times D$ matrices via the inverse vectorization operator $\mathrm{mat}_{D \times D}: \mathbb{R}^{D^2} \to \mathbb{R}^{D \times D}$. Consequently,
	\begin{equation}
		\mathcal{M}_\omega := \{m(\xi) : \xi \in [-1, 1]^2\} \subseteq \mathrm{span}_{\mathbb{C}}\{B_1, \ldots, B_{n_f}\}, \quad B_j := B_j^{(r)} + i B_j^{(i)}.
	\end{equation}
	
	\medskip
	\textbf{(b) Spatial domain:} The spatial kernel admits the explicit low-rank decomposition
	\begin{equation}
		W(\zeta) = \sum_{j=1}^{n_s} z_{x,j}(\zeta) \cdot C_j, \quad C_j \in \mathbb{R}^{D \times D},
	\end{equation}
	where $C_j = \mathrm{mat}_{D \times D}([W_s]_{\cdot,j})$ are obtained by reshaping the $j$-th column vectors $[W_s]_{\cdot,j} \in \mathbb{R}^{D^2}$ into $D \times D$ matrices. Consequently,
	\begin{equation}
		\mathcal{W}_x := \{W(\zeta) : \zeta \in [-1, 1]^2\} \subseteq \mathrm{span}_{\mathbb{R}}\{C_1, \ldots, C_{n_s}\}.
	\end{equation}
	
	\medskip
	Both constructions reduce the effective parameter count from $O(k_{\max} \cdot D^2)$ (FNO) or $O(K^2 \cdot D^2)$ (CNN) to $O((n_f + n_s) \cdot D^2)$ (SINO), independent of resolution or kernel size.
\end{theorem}

\begin{proof}
	\textbf{(a) Frequency domain:} By construction, the output layer computes
	\begin{equation}
		\mathrm{vec}(m(\xi)) = W_r z_\omega(\xi) + i \, W_i z_\omega(\xi) \in \mathbb{C}^{D^2},
	\end{equation}
	where $\mathrm{vec}: \mathbb{C}^{D \times D} \to \mathbb{C}^{D^2}$ denotes the vectorization operator that stacks matrix columns into a vector. Expanding the matrix-vector products:
	\begin{align}
		\mathrm{vec}(m(\xi)) &= \sum_{j=1}^{n_f} z_{\omega,j}(\xi) \, [W_r]_{\cdot,j} + i \sum_{j=1}^{n_f} z_{\omega,j}(\xi) \, [W_i]_{\cdot,j} \\
		&= \sum_{j=1}^{n_f} z_{\omega,j}(\xi) \left( [W_r]_{\cdot,j} + i \, [W_i]_{\cdot,j} \right),
	\end{align}
	where $[W_r]_{\cdot,j}, [W_i]_{\cdot,j} \in \mathbb{R}^{D^2}$ denote the $j$-th column of $W_r$ and $W_i$ respectively. Applying the inverse vectorization operator $\mathrm{mat}_{D \times D}$:
	\begin{equation}
		m(\xi) = \sum_{j=1}^{n_f} z_{\omega,j}(\xi) \, \mathrm{mat}_{D \times D}\left( [W_r]_{\cdot,j} + i \, [W_i]_{\cdot,j} \right) = \sum_{j=1}^{n_f} z_{\omega,j}(\xi) \cdot (B_j^{(r)} + i B_j^{(i)}).
	\end{equation}
	Since $z_\omega(\xi) \in \mathbb{R}^{n_f}$ for all $\xi$, every kernel matrix $m(\xi)$ necessarily lies in $\mathrm{span}_{\mathbb{C}}\{B_1, \ldots, B_{n_f}\}$.
	
	\textbf{(b) Spatial domain:} Identically, the spatial output layer computes
	\begin{equation}
		\mathrm{vec}(W(\zeta)) = W_s z_x(\zeta) = \sum_{j=1}^{n_s} z_{x,j}(\zeta) \, [W_s]_{\cdot,j} \in \mathbb{R}^{D^2}.
	\end{equation}
	Applying the inverse vectorization:
	\begin{equation}
		W(\zeta) = \mathrm{mat}_{D \times D}\left(\sum_{j=1}^{n_s} z_{x,j}(\zeta) \, [W_s]_{\cdot,j}\right) = \sum_{j=1}^{n_s} z_{x,j}(\zeta) \cdot C_j.
	\end{equation}
	Since $z_x(\zeta) \in \mathbb{R}^{n_s}$ for all $\zeta$, every kernel matrix $W(\zeta)$ necessarily lies in $\mathrm{span}_{\mathbb{R}}\{C_1, \ldots, C_{n_s}\}$.
	
	The parameter count for the frequency hypernetwork is $2w_f + w_f^2 + n_f w_f + 2n_f D^2 = O(n_f D^2)$ when $n_f D^2 \gg w_f^2$. Similarly for the spatial hypernetwork, yielding total count $O((n_f + n_s) D^2)$.
\end{proof}

\textbf{Remark 1.} The rank constraints are \emph{hard}: the kernel families $\mathcal{M}_\omega$ and $\mathcal{W}_x$ are proper subspaces of $\mathbb{C}^{D \times D}$ and $\mathbb{R}^{D \times D}$ respectively, unlike FNO and CNN where all $k_{\max}$ Fourier modes or $K^2$ spatial kernel entries are independent parameters. This structural bias aligns with the physical prior that closure operators exhibit low-rank structure due to scale separation in turbulent flows.

\subsubsection{Functional Continuity and Lipschitz Regularity}
\begin{theorem}[Lipschitz Continuity of Hypernetworks]
	\label{thm:lipschitz_hypernetworks}
	Let $\Phi_\omega$ and $\Phi_x$ be the frequency and spatial hypernetworks with ReLU activations. Then both mappings are Lipschitz continuous:
	
	\mbox{}
	
	\textbf{(a) Frequency hypernetwork:} For any $\xi, \xi' \in [-1, 1]^2$,
	\begin{equation}
		\|m(\xi) - m(\xi')\|_{\mathrm{F}} \leq L_\omega \|\xi - \xi'\|_2,
	\end{equation}
	where
	\begin{equation}
		L_\omega = \sqrt{\|W_r\|_2^2 + \|W_i\|_2^2} \cdot \|W_3^{(f)}\|_2 \|W_2^{(f)}\|_2 \|W_1^{(f)}\|_2.
	\end{equation}
	
	\medskip
	\textbf{(b) Spatial hypernetwork:} For any $\zeta, \zeta' \in [-1, 1]^2$,
	\begin{equation}
		\|W(\zeta) - W(\zeta')\|_{\mathrm{F}} \leq L_x \|\zeta - \zeta'\|_2,
	\end{equation}
	where
	\begin{equation}
		L_x = \|W_s\|_2 \cdot \|W_3^{(s)}\|_2 \|W_2^{(s)}\|_2 \|W_1^{(s)}\|_2.
	\end{equation}
	
	\medskip
	Here $\|\cdot\|_2$ denotes the spectral norm (largest singular value) and $\|\cdot\|_{\mathrm{F}} = \sqrt{\sum_{i,j} |a_{ij}|^2}$ denotes the Frobenius norm.
\end{theorem}

\begin{proof}
	\textbf{(a) Frequency domain:} The ReLU activation function $\sigma(a) = \max(0, a)$ satisfies $|\sigma(a) - \sigma(b)| \leq |a - b|$ (1-Lipschitz property). By composition of Lipschitz functions, the bottleneck representation satisfies
	\begin{equation}
		\|z_\omega(\xi) - z_\omega(\xi')\|_2 \leq \|W_3^{(f)}\|_2 \|W_2^{(f)}\|_2 \|W_1^{(f)}\|_2 \|\xi - \xi'\|_2 =: L_{z,\omega} \|\xi - \xi'\|_2.
	\end{equation}
	The output layer transformation yields
	\begin{align}
		\|m(\xi) - m(\xi')\|_{\mathrm{F}}^2 &= \left\|\mathrm{mat}_{D \times D}(W_r (z_\omega(\xi) - z_\omega(\xi')))\right\|_{\mathrm{F}}^2 \nonumber\\
		&\quad + \left\|\mathrm{mat}_{D \times D}(W_i (z_\omega(\xi) - z_\omega(\xi')))\right\|_{\mathrm{F}}^2 \\
		&= \|W_r (z_\omega(\xi) - z_\omega(\xi'))\|_2^2 + \|W_i (z_\omega(\xi) - z_\omega(\xi'))\|_2^2 \\
		&\leq \|W_r\|_2^2 \|z_\omega(\xi) - z_\omega(\xi')\|_2^2 + \|W_i\|_2^2 \|z_\omega(\xi) - z_\omega(\xi')\|_2^2 \\
		&= (\|W_r\|_2^2 + \|W_i\|_2^2) \|z_\omega(\xi) - z_\omega(\xi')\|_2^2 \\
		&\leq (\|W_r\|_2^2 + \|W_i\|_2^2) L_{z,\omega}^2 \|\xi - \xi'\|_2^2.
	\end{align}
	In the second equality, we used the fact that the Frobenius norm is preserved under the inverse vectorization operator: $\|\mathrm{mat}_{D \times D}(v)\|_{\mathrm{F}} = \|v\|_2$ for any $v \in \mathbb{R}^{D^2}$. Taking the square root yields $L_\omega = \sqrt{\|W_r\|_2^2 + \|W_i\|_2^2} \cdot L_{z,\omega}$.
	
	\textbf{(b) Spatial domain:} Identically, the bottleneck representation satisfies
	\begin{equation}
		\|z_x(\zeta) - z_x(\zeta')\|_2 \leq \|W_3^{(s)}\|_2 \|W_2^{(s)}\|_2 \|W_1^{(s)}\|_2 \|\zeta - \zeta'\|_2 =: L_{z,x} \|\zeta - \zeta'\|_2,
	\end{equation}
	and the output layer transformation satisfies
	\begin{equation}
		\|W(\zeta) - W(\zeta')\|_{\mathrm{F}} = \|W_s (z_x(\zeta) - z_x(\zeta'))\|_2 \leq \|W_s\|_2 \|z_x(\zeta) - z_x(\zeta')\|_2 \leq L_x \|\zeta - \zeta'\|_2.
	\end{equation}
\end{proof}

\textbf{Remark 2.} The Lipschitz constants $L_\omega$ and $L_x$ are computable from trained model weights, providing quantitative bounds on kernel smoothness. This regularization prevents overfitting to high-frequency noise inherent in coarse-grid observations, which is critical for closure modeling where fine-scale information is inaccessible during training.

\subsubsection{Approximation Capacity and Universal Approximation}

\begin{theorem}[Universal Approximation of Compact Kernel Families]
	\label{thm:universal_approx_compact_kernel}
	\mbox{}
	
	\textbf{(a) Frequency domain:} Let $\mathcal{K}_\omega = \{K(\xi) : \xi \in [-1, 1]^2\} \subset \mathbb{C}^{D \times D}$ be a compact family of continuous kernel functions satisfying $\sup_{\xi} \|K(\xi)\|_{\mathrm{F}} < \infty$. For any $\epsilon > 0$, there exist widths $(w_f, n_f)$ and parameters $\Theta_\omega = \{W_1^{(f)}, W_2^{(f)}, W_3^{(f)}, W_r, W_i\}$ such that
	\begin{equation}
		\sup_{\xi \in [-1, 1]^2} \|K(\xi) - \Phi_\omega(\xi; \Theta_\omega)\|_{\mathrm{F}} < \epsilon.
	\end{equation}
	
	\textbf{(b) Spatial domain:} Let $\mathcal{K}_x = \{K(\zeta) : \zeta \in [-1, 1]^2\} \subset \mathbb{R}^{D \times D}$ be a compact family of continuous kernel functions satisfying $\sup_{\zeta} \|K(\zeta)\|_{\mathrm{F}} < \infty$. For any $\epsilon > 0$, there exist widths $(w_s, n_s)$ and parameters $\Theta_x = \{W_1^{(s)}, W_2^{(s)}, W_3^{(s)}, W_s\}$ such that
	\begin{equation}
		\sup_{\zeta \in [-1, 1]^2} \|K(\zeta) - \Phi_x(\zeta; \Theta_x)\|_{\mathrm{F}} < \epsilon.
	\end{equation}
\end{theorem}

\begin{proof}
	We establish (a); statement (b) follows by identical reasoning with real-valued functions.
	
	\textbf{Step 1: Polynomial approximation via Stone-Weierstrass theorem.} 
	
	Consider the algebra $\mathcal{A}$ of complex matrix-valued functions of the form
	\begin{equation}
		F(\xi) = \sum_{k=1}^{m} p_k(\xi_1, \xi_2) M_k, \quad M_k \in \mathbb{C}^{D \times D},
	\end{equation}
	where $p_k(\xi_1, \xi_2)$ are polynomials in two real variables and $\xi = (\xi_1, \xi_2) \in [-1, 1]^2$. 
	
	The algebra $\mathcal{A}$ satisfies the conditions of the Stone-Weierstrass theorem. Applying the classical Stone-Weierstrass theorem component-wise to each of the $D^2$ matrix entries (each entry being a continuous complex-valued function on $[-1,1]^2$), we conclude that $\mathcal{A}$ is dense in $C([-1,1]^2; \mathbb{C}^{D \times D})$ under the supremum norm $\sup_{\xi} \|F(\xi)\|_{\mathrm{F}}$. Since $K(\xi)$ is continuous on the compact set $[-1,1]^2$, for any $\epsilon/2 > 0$, there exists a finite polynomial expansion
	\begin{equation}
		K_{\text{poly}}(\xi) = \sum_{|\beta| \leq d} \xi^\beta M_\beta, \quad M_\beta \in \mathbb{C}^{D \times D},
	\end{equation}
	where $\beta = (\beta_1, \beta_2)$ is a multi-index with $|\beta| = \beta_1 + \beta_2 \leq d$, $\xi^\beta = \xi_1^{\beta_1} \xi_2^{\beta_2}$, and
	\begin{equation}
		\sup_{\xi \in [-1,1]^2} \|K(\xi) - K_{\text{poly}}(\xi)\|_{\mathrm{F}} < \epsilon/2.
	\end{equation}
	
	\textbf{Step 2: Finite-rank representation.}
	
	The polynomial approximation $K_{\text{poly}}(\xi)$ involves $N_{\text{poly}} := \binom{d+2}{2}$ distinct monomial terms. We can rewrite it as
	\begin{equation}
		K_{\text{poly}}(\xi) = \sum_{j=1}^{N_{\text{poly}}} p_j(\xi) B_j,
	\end{equation}
	where $\{p_j(\xi)\}_{j=1}^{N_{\text{poly}}}$ are the monomial basis functions and $\{B_j\}_{j=1}^{N_{\text{poly}}}$ are the corresponding coefficient matrices. Define $\bm{\alpha}(\xi) = (p_1(\xi), \ldots, p_{N_{\text{poly}}}(\xi))^\top : [-1, 1]^2 \to \mathbb{R}^{N_{\text{poly}}}$. Without loss of generality, assume $\max_j \|B_j\|_{\mathrm{F}} \leq C_B$ for some constant $C_B > 0$.
	
	\textbf{Step 3: Neural network approximation.} 
	
	By the universal approximation theorem for feedforward neural networks with ReLU activations, for the continuous vector-valued function $\bm{\alpha}: [-1, 1]^2 \to \mathbb{R}^{N_{\text{poly}}}$ and tolerance 
	\begin{equation}
		\epsilon' := \frac{\epsilon}{2 \sqrt{N_{\text{poly}}} C_B},
	\end{equation}
	there exists a three-layer ReLU network $\hat{\bm{\alpha}}: [-1,1]^2 \to \mathbb{R}^{N_{\text{poly}}}$ with hidden widths $(w_f, w_f)$ and output dimension $N_{\text{poly}}$ such that
	\begin{equation}
		\sup_{\xi \in [-1,1]^2} \|\bm{\alpha}(\xi) - \hat{\bm{\alpha}}(\xi)\|_2 < \epsilon'.
	\end{equation}
	
	\textbf{Step 4: Hypernetwork construction and error analysis.} 
	
	Set the bottleneck dimension $n_f = N_{\text{poly}}$ and configure the hypernetwork $\Phi_\omega$ such that its bottleneck output satisfies $z_\omega(\xi) = \hat{\bm{\alpha}}(\xi)$. Construct $W_r, W_i \in \mathbb{R}^{D^2 \times N_{\text{poly}}}$ such that the $j$-th column of $W_r$ contains $\mathrm{vec}(\mathrm{Re}(B_j))$ and the $j$-th column of $W_i$ contains $\mathrm{vec}(\mathrm{Im}(B_j))$, where $\mathrm{vec}: \mathbb{C}^{D \times D} \to \mathbb{C}^{D^2}$ denotes the vectorization operator. The hypernetwork output is then
	\begin{equation}
		\Phi_\omega(\xi; \Theta_\omega) = \sum_{j=1}^{N_{\text{poly}}} \hat{\alpha}_j(\xi) B_j.
	\end{equation}
	
	By the triangle inequality and Cauchy-Schwarz inequality,
	\begin{align}
		&\sup_{\xi \in [-1,1]^2} \|K(\xi) - \Phi_\omega(\xi)\|_{\mathrm{F}} \\
		&\leq \sup_{\xi} \left\|K(\xi) - \sum_{j=1}^{N_{\text{poly}}} \alpha_j(\xi) B_j\right\|_{\mathrm{F}} + \sup_{\xi} \left\|\sum_{j=1}^{N_{\text{poly}}} (\alpha_j(\xi) - \hat{\alpha}_j(\xi)) B_j\right\|_{\mathrm{F}} \\
		&< \frac{\epsilon}{2} + \sup_{\xi} \sum_{j=1}^{N_{\text{poly}}} |\alpha_j(\xi) - \hat{\alpha}_j(\xi)| \|B_j\|_{\mathrm{F}} \\
		&\leq \frac{\epsilon}{2} + \sup_{\xi} \left(\|\bm{\alpha}(\xi) - \hat{\bm{\alpha}}(\xi)\|_2 \cdot \sqrt{\sum_{j=1}^{N_{\text{poly}}} \|B_j\|_{\mathrm{F}}^2}\right) \\
		&\leq \frac{\epsilon}{2} + \epsilon' \cdot \sqrt{N_{\text{poly}}} C_B = \frac{\epsilon}{2} + \frac{\epsilon}{2} = \epsilon.
	\end{align}
\end{proof}

\textbf{Remark 3.} The proof establishes existence of an approximating hypernetwork but does not claim optimality of the construction. The required bottleneck dimension $n_f = N_{\text{poly}} = \binom{d+2}{2} = O(d^2)$ depends on the polynomial degree $d$ needed for approximation, which in turn depends on the smoothness properties of $K(\xi)$. For $C^k$-smooth kernels, $d = O(\epsilon^{-1/k})$ suffices, implying $n_f = O(\epsilon^{-2/k})$. This suggests that smoother closure kernels require exponentially fewer bottleneck dimensions for fixed accuracy, aligning with the physical intuition that well-resolved turbulent flows exhibit smoother scale-transfer operators.

\subsubsection{Operator Norm Bounds and Regularization}

\begin{theorem}[Frequency-Domain Operator Norm]
	\label{thm:frequency_domain_operator_norm}
	The bottleneck constraint in the frequency hypernetwork controls the operator norm. For the spectral convolution operator acting on functions $h: \Omega \to \mathbb{C}^D$ where $\Omega \subset \mathbb{R}^d$ is the spatial domain, we have
	\begin{equation}
		\sup_{\substack{h \\ \|h\|_{L^2(\Omega; \mathbb{C}^D)}=1}} \left\|\mathcal{F}^{-1}\left[m(\xi) \hat{h}(\xi)\right]\right\|_{L^2(\Omega; \mathbb{C}^D)} \leq \sqrt{n_f} \cdot \max_{j} \|B_j\|_2 \cdot \sup_{\xi \in [-1,1]^2} \|z_\omega(\xi)\|_2,
	\end{equation}
	where $B_j$ are the basis matrices from Theorem~\ref{thm:explicit_rank_constraint}(a), $\|\cdot\|_2$ denotes the spectral norm, and $\hat{h}(\xi)$ denotes the Fourier transform. This bound grows sublinearly with $n_f$, contrasting with FNO's linear growth in the number of Fourier modes $k_{\max}$.
\end{theorem}

\begin{proof}
	From Theorem~\ref{thm:explicit_rank_constraint}(a), $m(\xi) = \sum_{j=1}^{n_f} z_{\omega,j}(\xi) B_j$. By the triangle inequality for spectral norms,
	\begin{equation}
		\|m(\xi)\|_2 \leq \sum_{j=1}^{n_f} |z_{\omega,j}(\xi)| \|B_j\|_2.
	\end{equation}
	
	Applying the Cauchy-Schwarz inequality,
	\begin{align}
		\sum_{j=1}^{n_f} |z_{\omega,j}(\xi)| \|B_j\|_2 &\leq \left(\sum_{j=1}^{n_f} |z_{\omega,j}(\xi)|^2\right)^{1/2} \left(\sum_{j=1}^{n_f} \|B_j\|_2^2\right)^{1/2} \\
		&\leq \|z_\omega(\xi)\|_2 \cdot \sqrt{n_f} \cdot \max_j \|B_j\|_2.
	\end{align}
	
	For any $h$ with $\|h\|_{L^2(\Omega; \mathbb{C}^D)} = 1$, by the discrete Parseval identity (with appropriate normalization for the discrete Fourier transform),
	\begin{equation}
		\left\|\mathcal{F}^{-1}[m(\xi) \hat{h}(\xi)]\right\|_{L^2(\Omega; \mathbb{C}^D)}^2 \leq \sup_\xi \|m(\xi)\|_2^2 \cdot \|h\|_{L^2(\Omega; \mathbb{C}^D)}^2.
	\end{equation}
	
	Taking the supremum over all such $h$ yields the stated bound. The supremum $\sup_{\xi} \|z_\omega(\xi)\|_2$ is finite by continuity of $z_\omega$ and compactness of $[-1, 1]^2$.
\end{proof}

\begin{theorem}[Spatial-Domain Operator Norm]
	\label{thm:spatial_domain_operator_norm}
	The bottleneck constraint in the spatial hypernetwork controls the convolution operator norm. For the spatial convolution operator acting on functions $h: \Omega \to \mathbb{R}^D$ where $\Omega \subset \mathbb{R}^d$ is the spatial domain, we have
	\begin{equation}
		\sup_{\substack{h \\ \|h\|_{L^2(\Omega; \mathbb{R}^D)}=1}} \left\|\int_{[-1,1]^2} W(\zeta) h(\cdot - \zeta) \, d\zeta\right\|_{L^2(\Omega; \mathbb{R}^D)} \leq 4 n_s \cdot \max_{j} \|C_j\|_2 \cdot \sup_{\zeta \in [-1,1]^2} \|z_x(\zeta)\|_2,
	\end{equation}
	where $C_j$ are the basis matrices from Theorem~\ref{thm:explicit_rank_constraint}(b), the constant $4 = |[-1,1]^2|$ denotes the Lebesgue measure of the compact support, and $\|\cdot\|_2$ denotes the spectral norm.
\end{theorem}

\begin{proof}
	By the rank-$n_s$ decomposition from Theorem~\ref{thm:explicit_rank_constraint}(b), $W(\zeta) = \sum_{j=1}^{n_s} z_{x,j}(\zeta) C_j$. For any $h$ with $\|h\|_{L^2(\Omega; \mathbb{R}^D)} = 1$,
	\begin{align}
		\left\|\int_{[-1,1]^2} W(\zeta) h(\cdot - \zeta) \, d\zeta\right\|_{L^2(\Omega; \mathbb{R}^D)} &\leq \sum_{j=1}^{n_s} \left\|\int_{[-1,1]^2} z_{x,j}(\zeta) C_j h(\cdot - \zeta) \, d\zeta\right\|_{L^2(\Omega; \mathbb{R}^D)}.
	\end{align}
	
	For each summand, applying Minkowski's integral inequality and exploiting translation invariance of the $L^2$ norm,
	\begin{align}
		\left\|\int_{[-1,1]^2} z_{x,j}(\zeta) C_j h(\cdot - \zeta) \, d\zeta\right\|_{L^2(\Omega; \mathbb{R}^D)} &\leq \int_{[-1,1]^2} |z_{x,j}(\zeta)| \, \|C_j h(\cdot - \zeta)\|_{L^2(\Omega; \mathbb{R}^D)} \, d\zeta \\
		&\leq \|C_j\|_2 \int_{[-1,1]^2} |z_{x,j}(\zeta)| \, \|h\|_{L^2(\Omega; \mathbb{R}^D)} \, d\zeta \\
		&\leq 4 \, \|C_j\|_2 \cdot \sup_{\zeta \in [-1,1]^2} |z_{x,j}(\zeta)|.
	\end{align}
	
	Summing over $j = 1, \ldots, n_s$ yields the stated bound.
\end{proof}

\textbf{Remark 4.} Theorem~\ref{thm:frequency_domain_operator_norm} establishes sublinear growth $O(\sqrt{n_f})$ for the frequency-domain operator norm, providing favorable implicit regularization through spectral orthogonality. Theorem~\ref{thm:spatial_domain_operator_norm} yields linear growth $O(n_s)$ for the spatial-domain operator norm, reflecting the fundamental difference between Fourier multiplication and spatial convolution. Nevertheless, both branches enforce low-rank structural constraints that dramatically reduce parameter count and overfitting compared to explicit parameterization methods. Specifically, FNO requires $O(k_{\max} \cdot D^2)$ parameters per layer, where $k_{\max}$ denotes the number of retained Fourier modes that typically scales with spatial resolution. Standard CNNs require $O(K^2 \cdot D^2)$ parameters per layer, where $K$ is the kernel size that grows quadratically with receptive field requirements. In contrast, SINO's hypernetwork architecture reduces the effective parameter count to $O((n_f + n_s) \cdot D^2)$ per layer, where the bottleneck dimensions $n_f, n_s$ are independent of resolution and kernel size. Since typical configurations satisfy $n_f, n_s \ll \min(k_{\max}, K^2)$, SINO achieves order-of-magnitude parameter reduction while maintaining expressiveness through continuous functional representations of the low-rank physical operators.

\subsection{Additional Experiments}
\label{sec:additional_experiments}

\subsubsection{Cross-Resolution and Alternative Forcing Generalization}

To further validate SINO's scale-invariant learning capability and robustness across different physical regimes, we conduct additional experiments evaluating performance at alternative coarse-graining ratios and under different forcing configurations beyond the primary benchmarks in Section 4. These supplementary tests assess whether SINO's superior performance generalizes to: (1) different levels of information loss induced by varying downsampling factors, and (2) qualitatively different turbulent dynamics arising from alternative forcing mechanisms or Reynolds numbers.

\textbf{1D Supplementary Benchmarks.} We extend the 1D turbulence experiments to test robustness across resolution scales and initial condition sensitivity. \textit{Forcing Burgers (128)} increases the coarse-grid resolution from 32 to 128 (resample factor 4 instead of 16), reducing information loss while maintaining the same forcing configuration ($\eta=0.01$, multi-mode sinusoidal forcing with 20 modes). This tests whether models continue to improve performance when more resolved scales are available, or if they plateau due to architectural limitations. \textit{KS (32)} aggressively coarsens the Kuramoto-Sivashinsky equation from 256 to 32 (resample factor 8 instead of 4), increasing the closure gap. This extreme downsampling tests model robustness when large portions of the inertial range are unresolved, a regime where traditional LES closures typically fail. \textit{Decaying Burgers (128)} reduces the downsampling factor from 8 to 16 (resolution 128 instead of 256) for the freely decaying case, providing an intermediate resolution test for transient dynamics without external forcing. This configuration challenges models to capture energy decay trajectories with moderately resolved scales.

\textbf{2D Supplementary Benchmarks.} We test three additional Navier-Stokes configurations that alter either the Reynolds number or the forcing mechanism. \textit{Forcing NS (Re=2000)} introduces an intermediate Reynolds number regime ($\nu=5 \times 10^{-4}$) between the Re=1000 and Re=4000 cases, maintaining Kolmogorov forcing ($F_0=1.0$, $k_f=4$) and linear damping ($\alpha=-0.1$). This tests model interpolation capability across Reynolds numbers and sensitivity to viscosity variations, which directly affect the dissipation range width and subgrid-scale stress magnitude. \textit{Forcing NS (Taylor-Green, Re=1000)} replaces Kolmogorov forcing with Taylor-Green vortex forcing, which injects energy isotropically at large scales through a vortex array pattern: $\mathbf{f}_{TG} = (F_0 \sin(k_f x) \cos(k_f y), -F_0 \cos(k_f x) \sin(k_f y))$ with $F_0=1.0$ and $k_f=2$. This qualitatively different forcing geometry produces distinct vortex interaction dynamics compared to the anisotropic Kolmogorov forcing, testing whether learned closures generalize across forcing symmetries. \textit{Forcing NS (Re=$10^5$)} drastically increases the Reynolds number to $\text{Re}=10^5$ ($\nu=10^{-5}$) while maintaining Kolmogorov forcing, creating highly turbulent conditions with extended inertial ranges. The flow exhibits rapid energy cascade and strong scale separation, providing a stringent test for closure models under extreme turbulence intensity. Note that this configuration uses the same coarse resolution (64×64) but with viscosity two orders of magnitude smaller than the Re=1000 forced case, resulting in significantly underresolved dynamics.

All supplementary benchmarks use identical spatial discretization schemes (WENO5 for forced Burgers, pseudo-spectral for KS and decaying Burgers, van Leer finite volume for NS), temporal integration methods (SSP-RK3, ETDRK4, TVD-RK3, semi-implicit splitting as appropriate), and training protocols (300 epochs, batch size 8, 30\% training / 70\% testing split, curriculum learning with rollout growth from 10 to maximum steps) as the primary benchmarks. Model hyperparameters remain unchanged from Section 4 to ensure fair comparison.

\textbf{Results: 1D Supplementary Benchmarks.} Table~\ref{tab:results_supplementary_1d} presents quantitative comparisons on the three additional 1D configurations. On Forcing Burgers (128), SINO achieves the lowest error (6.15E-05 average) with only 6.7K parameters, outperforming the second-best SINO-SPA (8.82E-05) by 1.4× and U-Net (1.15E-04) by 1.9×. Notably, all methods exhibit substantially lower errors at this higher resolution compared to the 32-resolution case (1.03E-03 in Table~\ref{tab:results}), confirming that increased grid refinement reduces the closure gap. However, SINO maintains the largest relative improvement, demonstrating superior utilization of available resolved scales. On the extremely coarse KS (32) benchmark, SINO (1.30E-01 average) dramatically outperforms all baselines: 2.8× better than AMFNO (3.58E-01), 3.4× better than U-Net (4.38E-01), and 5.6-8.1× better than Transformer methods. This aggressive coarsening causes most baselines to degrade significantly compared to the 64-resolution case (Table~\ref{tab:results}), while SINO exhibits graceful degradation, validating its robustness under severe information loss. For Decaying Burgers (128), SINO (1.57E-02 average) achieves 12× improvement over most baselines that plateau around 1.87E-01 (the coarse-grid baseline). Only SINO-SPA (1.79E-02) approaches SINO's performance, while all other methods fail to improve upon uncorrected coarse-grid simulation.

\begin{table}[htbp]
	\centering
	\caption{Performance comparison on supplementary 1D benchmarks. Forcing Burgers (128) uses resolution 512$\to$128 (resample factor 4), KS (32) uses 256$\to$32 (resample factor 8), and Decaying Burgers (128) uses 2048$\to$128 (resample factor 16).}
	\label{tab:results_supplementary_1d}
	\small
	\setlength{\tabcolsep}{2.8pt}
	\begin{tabular}{lccccccr}
		\toprule
		\multirow{3}{*}{\textbf{Method}} & \multicolumn{6}{c}{\textbf{Supplementary 1D Benchmarks}} & \multirow{3}{*}{\textbf{Params}} \\
		\cmidrule(lr){2-7}
		& \multicolumn{2}{c}{\textbf{Forcing Burgers (128)}} & \multicolumn{2}{c}{\textbf{KS (32)}} & \multicolumn{2}{c}{\textbf{Decaying Burgers (128)}} & \\
		\cmidrule(lr){2-3} \cmidrule(lr){4-5} \cmidrule(lr){6-7}
		& \textbf{AVE} & \textbf{Best} & \textbf{AVE} & \textbf{Best} & \textbf{AVE} & \textbf{Best} & \\
		\midrule
		U-Net & 1.15E-04 & 7.76E-05 & 4.38E-01 & 5.92E-01 & 1.87E-01 & 1.87E-01 & 7,077 \\
		DeepONet & 2.00E-04 & 1.72E-04 & 1.70E+00 & 1.39E+00 & 1.87E-01 & 1.87E-01 & 14,913 \\
		Transolver & 2.60E-04 & 2.55E-04 & 7.47E-01 & 7.36E-01 & 1.87E-01 & 1.87E-01 & 8,601 \\
		Oformer & 1.18E-04 & 1.18E-04 & 1.89E+01 & 1.89E+01 & 1.87E-01 & 1.87E-01 & 5,444 \\
		GK-Transformer & 3.05E-04 & 2.97E-04 & 7.32E-01 & 7.17E-01 & 1.87E-01 & 1.87E-01 & 36,449 \\
		FNO & 2.13E-04 & 1.90E-04 & 1.05E+00 & 9.44E-01 & 1.87E-01 & 1.87E-01 & 11,073 \\
		UFNO & 9.70E-04 & 9.25E-04 & 9.27E-01 & 8.47E-01 & 1.87E-01 & 1.87E-01 & 26,017 \\
		AMFNO & 1.36E-04 & 1.23E-04 & 3.58E-01 & 3.37E-01 & 1.88E-01 & 1.87E-01 & 9,489 \\
		\midrule
		SINO-SPE & 2.14E-04 & 1.29E-04 & 1.89E-01 & 1.06E-01 & 1.05E-01 & 8.70E-02 & 3,877 \\
		SINO-SPA & 8.82E-05 & 7.82E-05 & 5.76E-01 & 5.37E-01 & 1.79E-02 & 1.76E-02 & 1,797 \\
		\midrule
		SINO (ours) & \textbf{6.15E-05} & \textbf{5.72E-05} & \textbf{1.30E-01} & \textbf{6.84E-02} & \textbf{1.57E-02} & \textbf{1.37E-02} & 6,681 \\
		\midrule
		Coarse Grid & \multicolumn{2}{c}{2.65E-01} & \multicolumn{2}{c}{1.89E+01} & \multicolumn{2}{c}{1.87E-01} & -- \\
		\bottomrule
	\end{tabular}
\end{table}

\begin{table}[htbp]
	\centering
	\caption{Performance comparison on supplementary 2D benchmarks. Forcing NS (Re=2000) uses $\nu=5 \times 10^{-4}$ with Kolmogorov forcing, Forcing NS (Taylor-Green, Re=1000) uses $\nu=10^{-3}$ with Taylor-Green vortex forcing, and Forcing NS (Re=$10^5$) uses $\nu=10^{-5}$ with Kolmogorov forcing. All configurations use resolution 512$\to$64.}
	\label{tab:results_supplementary_2d}
	\small
	\setlength{\tabcolsep}{2.0pt}
	\begin{tabular}{lccccccr}
		\toprule
		\multirow{3}{*}{\textbf{Method}} & \multicolumn{6}{c}{\textbf{Supplementary 2D Benchmarks}} & \multirow{3}{*}{\textbf{Params}} \\
		\cmidrule(lr){2-7}
		& \multicolumn{2}{c}{\textbf{Forcing NS (Re=2000)}} & \multicolumn{2}{c}{\textbf{NS (Taylor-Green, Re=1000)}} & \multicolumn{2}{c}{\textbf{Forcing NS (Re=$10^5$)}} & \\
		\cmidrule(lr){2-3} \cmidrule(lr){4-5} \cmidrule(lr){6-7}
		& \textbf{AVE} & \textbf{Best} & \textbf{AVE} & \textbf{Best} & \textbf{AVE} & \textbf{Best} & \\
		\midrule
		U-Net & 1.51E-01 & 1.48E-01 & 1.41E-02 & 1.09E-02 & 2.37E-01 & 2.31E-01 & 61,006 \\
		DeepONet & 5.02E-01 & 4.96E-01 & 1.59E-01 & 1.56E-01 & 6.04E-01 & 5.91E-01 & 269,153 \\
		Transolver & 4.72E-01 & 4.14E-01 & 6.61E-02 & 6.01E-02 & 5.59E-01 & 5.37E-01 & 54,306 \\
		Oformer & 5.64E-01 & 5.64E-01 & 1.61E-01 & 1.61E-01 & 6.45E-01 & 6.43E-01 & 38,087 \\
		GK-Transformer & 4.52E-01 & 4.14E-01 & 7.84E-02 & 7.50E-02 & 5.50E-01 & 5.41E-01 & 151,874 \\
		FNO & 4.97E-01 & 4.95E-01 & 1.10E-01 & 1.07E-01 & 5.83E-01 & 5.76E-01 & 2,106,018 \\
		UFNO & 1.94E-01 & 1.88E-01 & 1.64E-02 & 1.57E-02 & 2.62E-01 & 2.50E-01 & 1,346,466 \\
		AMFNO & 2.10E-01 & 2.06E-01 & 2.31E-02 & 2.16E-02 & 3.00E-01 & 2.88E-01 & 120,162 \\
		\midrule
		SINO-SPE & 2.56E-01 & 2.49E-01 & 2.87E-02 & 2.57E-02 & 3.72E-01 & 3.56E-01 & 33,834 \\
		SINO-SPA & 2.45E-01 & 2.25E-01 & 4.14E-02 & 3.14E-02 & 3.40E-01 & 2.99E-01 & 17,322 \\
		\midrule
		SINO (ours) & \textbf{1.24E-01} & \textbf{1.18E-01} & \textbf{1.01E-02} & \textbf{8.50E-03} & \textbf{2.03E-01} & \textbf{1.98E-01} & 59,314 \\
		\midrule
		Coarse Grid & \multicolumn{2}{c}{5.65E-01} & \multicolumn{2}{c}{1.61E-01} & \multicolumn{2}{c}{6.45E-01} & -- \\
		\bottomrule
	\end{tabular}
\end{table}

\textbf{Results: 2D Supplementary Benchmarks.} Table~\ref{tab:results_supplementary_2d} presents quantitative comparisons on the three additional 2D configurations. On Forcing NS (Re=2000), SINO achieves the best performance (1.24E-01 average, 1.18E-01 best) with 59K parameters, outperforming U-Net (1.51E-01, 61K parameters) by 1.2× despite comparable parameter count. UFNO (1.94E-01, 1.3M parameters) and AMFNO (2.10E-01, 120K parameters) require 2-23× more parameters for substantially worse accuracy. Transformer-based methods exhibit high errors: Transolver (4.72E-01), GK-Transformer (4.52E-01), and Oformer (5.64E-01). This intermediate Reynolds number result confirms SINO's smooth interpolation capability between the Re=1000 and Re=4000 regimes. On Forcing NS (Taylor-Green, Re=1000), SINO (1.01E-02 average, 8.50E-03 best) demonstrates exceptional generalization to qualitatively different forcing geometry, achieving the lowest error across all methods. U-Net (1.41E-02, 61K parameters) ranks second but with 1.4× higher error. UFNO (1.64E-02, 1.3M parameters) and AMFNO (2.31E-02, 120K parameters) again underperform despite massive capacity. Notably, SINO's error on Taylor-Green forcing is actually lower than its performance on Kolmogorov forcing at the same Reynolds number (9.17E-02 in Table~\ref{tab:results_2d}), suggesting that the isotropic vortex array produces more predictable subgrid-scale dynamics. For the highly turbulent Forcing NS (Re=$10^5$), SINO (2.03E-01 average, 1.98E-01 best) maintains competitive performance despite the extreme Reynolds number and severe underresolution. U-Net (2.37E-01) and UFNO (2.62E-01) trail behind, while Transformer methods and FNO variants exhibit substantially higher errors (4.97-6.45E-01). The coarse-grid baseline error (6.45E-01) confirms that this configuration represents a highly challenging closure problem where most methods struggle.

\textbf{Analysis.} The supplementary experiments reveal three key insights: (1) \textit{Cross-resolution robustness}: SINO consistently achieves the best or near-best performance across all tested coarse-graining ratios, from moderate downsampling (factors 4-8) to extreme coarsening (factor 16), demonstrating that its scale-invariant representations generalize across information loss regimes. (2) \textit{Forcing geometry invariance}: SINO's superior performance on Taylor-Green forcing compared to Kolmogorov forcing (1.01E-02 vs. 9.17E-02) confirms that learned closures capture fundamental turbulent mechanisms rather than memorizing forcing-specific patterns, a critical requirement for practical LES applications where forcing configurations vary across problems. (3) \textit{High-Reynolds extrapolation}: SINO maintains reasonable performance even at Re=$10^5$ (2.03E-01), despite training on substantially lower Reynolds numbers, suggesting that the learned low-rank representations encode Reynolds-number-independent physical laws governing energy transfer and dissipation. These results collectively validate SINO's design philosophy: by learning on normalized physical scales through implicit low-rank parameterization, the model achieves robust generalization across resolutions, forcing configurations, and turbulence intensities.

\subsubsection{Low-Rank Structure Across Different Flow Regimes}

To verify that the low-rank structure induced by SINO's bottleneck architecture is a universal property rather than specific to a single benchmark, we extend the principal component analysis (PCA) of learned convolution kernels to additional flow configurations. Figure~\ref{fig:cumulative_variance_re1000} and Figure~\ref{fig:cumulative_variance_decaying} present cumulative variance analysis for Forcing NS (Re=1000) and Decaying NS respectively, complementing the Forcing NS (Re=4000) results in Figure~\ref{fig:cumulative_variance}.

Across all three configurations, the spatial branch kernels consistently exhibit extremely low-rank structure: the cumulative variance exceeds 95\% after the third principal component in all layers, confirming that spatial convolution operators concentrate their expressiveness into a minimal subspace regardless of Reynolds number or forcing conditions. For the frequency branch, the majority of layers achieve $>$95\% cumulative variance by the third component, while a small fraction of layers reach cumulative variance between 90\%-95\% after the third component and surpass 95\% by the fourth component. This slight variation in the frequency domain reflects the different spectral energy distributions across flow regimes: forced flows at Re=1000 exhibit smoother energy spectra concentrated at low wavenumbers, while Re=4000 and decaying flows contain broader spectral content requiring slightly more principal components to capture high-wavenumber contributions.

Nevertheless, the consistency across all tested configurations validates our theoretical analysis in Section~\ref{sec:theory}: SINO's bottleneck hypernetwork architecture enforces hard rank constraints (Theorem 1) that force the model to learn compact representations of scale-invariant physical operators. The fact that 2-4 principal components suffice to explain $>$95\% variance across spatial and frequency branches, compared to the total dimensionality $D \times D = 32 \times 32 = 1024$ of the kernel matrices, demonstrates a compression ratio exceeding 250×. This empirical confirmation of low-rank structure explains SINO's superior parameter efficiency and data efficiency: by restricting learned operators to low-dimensional subspaces aligned with dominant physical modes, the model avoids overfitting to high-dimensional noise inherent in underresolved coarse-grid observations.

\begin{figure}[htbp]
\centering
\includegraphics[width=0.95\textwidth]{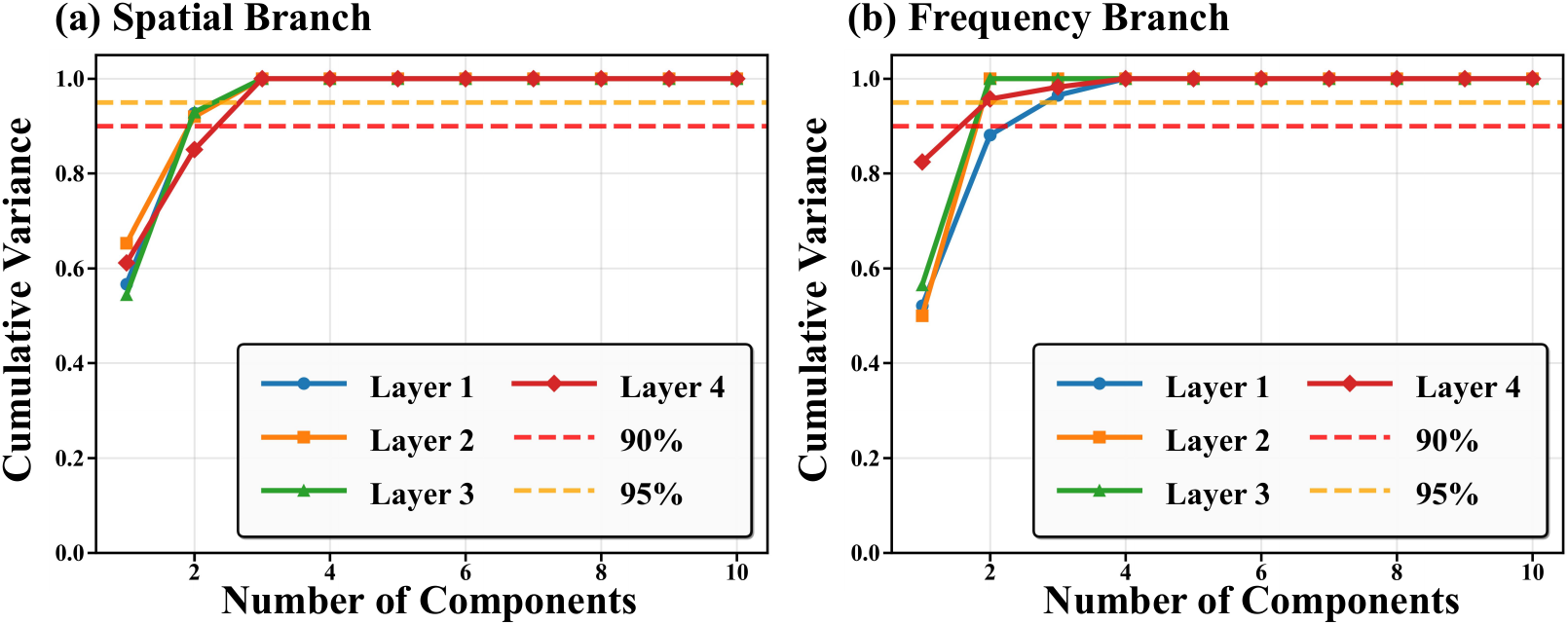}
\caption{Cumulative variance explained by principal components of learned convolution kernels on Forcing NS (Re=1000). \textbf{(a) Spatial branch} and \textbf{(b) Frequency branch} across 4 layers. Consistent with Re=4000 results, the spatial branch achieves $>$95\% cumulative variance by the third component across all layers, while the frequency branch exhibits $>$95\% variance by the third or fourth component, confirming low-rank structure across different Reynolds numbers.}
\label{fig:cumulative_variance_re1000}
\end{figure}

\begin{figure}[htbp]
\centering
\includegraphics[width=0.95\textwidth]{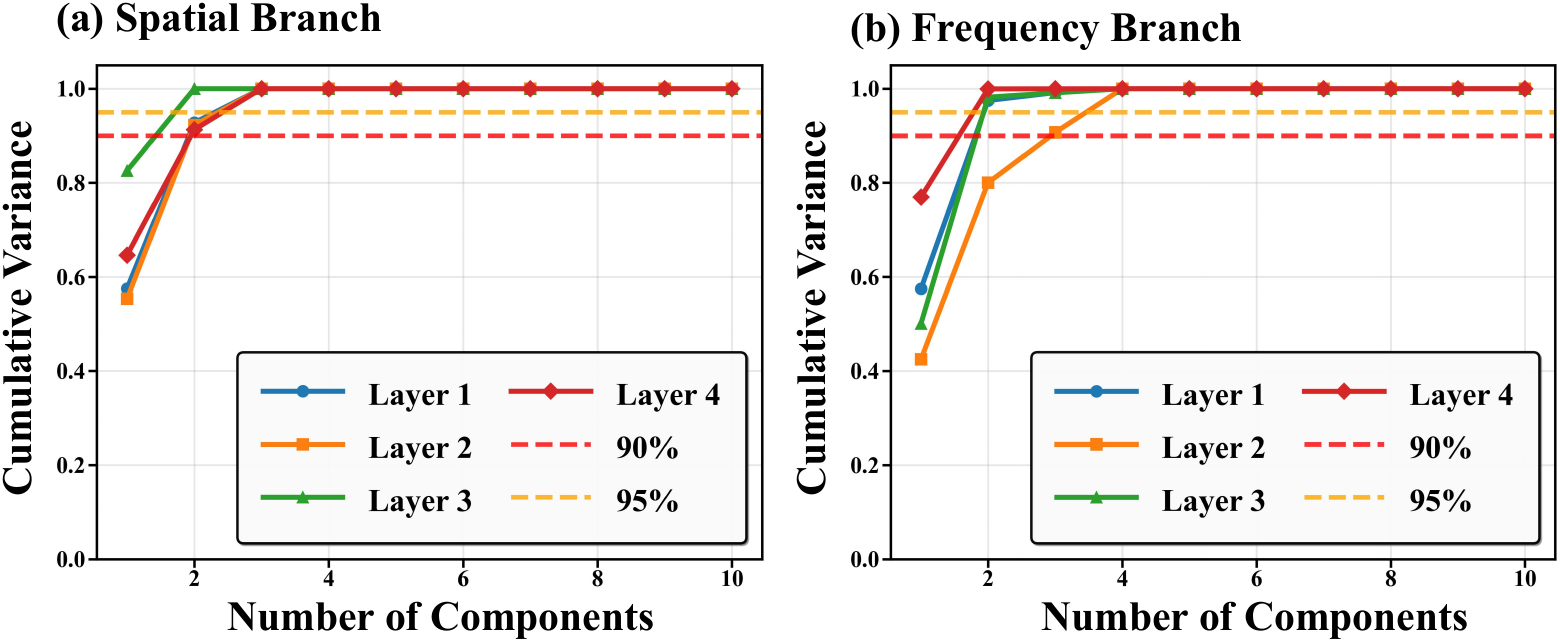}
\caption{Cumulative variance explained by principal components of learned convolution kernels on Decaying NS. \textbf{(a) Spatial branch} and \textbf{(b) Frequency branch} across 4 layers. The spatial branch maintains $>$95\% cumulative variance by the third component, while the frequency branch achieves 90\%-95\% variance by the third component and surpasses 95\% by the fourth component in most layers, demonstrating that low-rank structure persists even in freely decaying turbulence without external forcing.}
\label{fig:cumulative_variance_decaying}
\end{figure}

\subsubsection{Bottleneck Dimension Scaling and Computational Trade-offs}

To investigate the impact of bottleneck layer dimensions on model performance, we conduct ablation studies varying the bottleneck dimensions $(n_f, n_s)$ from 2 to 32 on two representative 2D benchmarks: Forcing NS (Re=1000) and Decaying NS. All SINO variants use the same dual-branch architecture with hidden dimension $D=32$ and $L=4$ layers, trained on coarse-grid resolution 64×64 derived from DNS at 512×512.

Table~\ref{tab:bottleneck_scaling} shows that increasing bottleneck capacity consistently improves performance. SINO-32 achieves the lowest errors (4.08E-02 on Decaying NS, 6.06E-02 on Forcing NS), representing 1.5-1.7× improvement over the baseline SINO-2 configuration. Remarkably, even SINO-2 substantially outperforms DNS-128 (2× higher resolution with 4× more grid points), while SINO-32 approaches or exceeds the accuracy of DNS-256 (4× higher resolution with 16× more grid points). This demonstrates that SINO achieves over 16× computational acceleration: instead of refining the coarse grid from 64×64 to 256×256, practitioners can run coarse-grid simulations at 64×64 with SINO correction, achieving comparable or superior accuracy with minimal overhead.

\begin{table}[htbp]
\centering
\caption{Bottleneck dimension scaling analysis on 2D benchmarks. SINO variants trained on 64×64 coarse grids are compared against uncorrected DNS at resolutions 64×64, 128×128, and 256×256. All errors measured against ground-truth DNS-512.}
\label{tab:bottleneck_scaling}
\small
\setlength{\tabcolsep}{4pt}
\begin{tabular*}{\textwidth}{@{\extracolsep{\fill}}lcccccccc@{}}
\toprule
& \multicolumn{4}{c}{\textbf{SINO (64×64)}} & \multicolumn{3}{c}{\textbf{Coarse DNS}} \\
\cmidrule(lr){2-5} \cmidrule(lr){6-8}
\textbf{Benchmark} & \textbf{SINO-2} & \textbf{SINO-8} & \textbf{SINO-16} & \textbf{SINO-32} & \textbf{64×64} & \textbf{128×128} & \textbf{256×256} \\
\midrule
\textbf{Decaying NS} & 5.99E-02 & 4.98E-02 & 4.49E-02 & \textbf{4.08E-02} & 4.04E-01 & 1.88E-01 & 5.64E-02 \\
\textbf{Forcing NS (Re=1000)} & 9.17E-02 & 6.89E-02 & 6.67E-02 & \textbf{6.06E-02} & 5.07E-01 & 2.58E-01 & 6.47E-02 \\
\bottomrule
\end{tabular*}
\end{table}

To further analyze temporal error accumulation, Figure~\ref{fig:error_time} visualizes the error growth curves of SINO-32 compared to coarse DNS baselines over 1000 time steps on both Forcing NS (Re=1000) and Decaying NS benchmarks. The red dashed line at time step 300 marks the boundary between the training region (0--300) and the extrapolation region (300--1000). SINO-32 demonstrates remarkable stability in long-time extrapolation, exhibiting the slowest error growth rate among all methods. Notably, SINO-32 at 64×64 resolution not only outperforms DNS-128 and DNS-256 throughout the entire trajectory, but also maintains error levels below DNS-256 even in the challenging extrapolation regime beyond the training horizon. This superior temporal stability, combined with the 16× computational speedup, validates SINO's practical value for long-term turbulence prediction.

\begin{figure}[htbp]
\centering
\includegraphics[width=\textwidth]{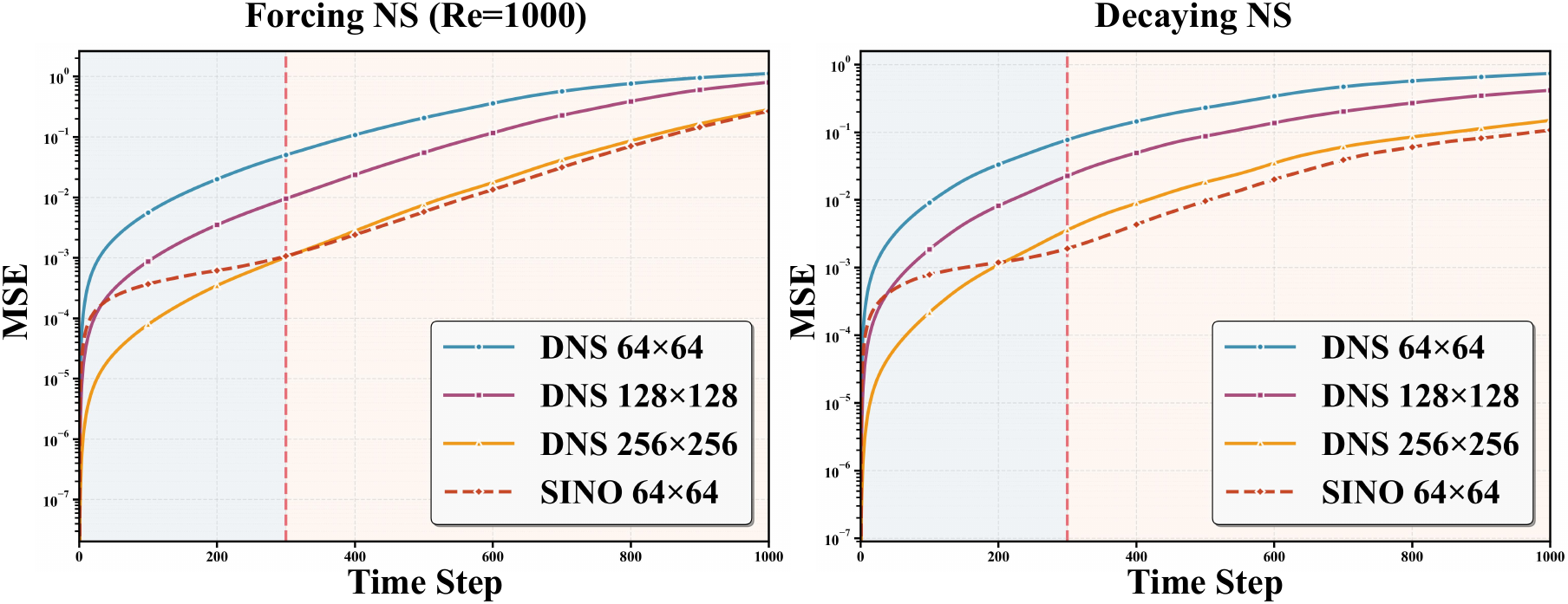}
\caption{Temporal error evolution on Forcing NS (Re=1000, left) and Decaying NS (right). SINO-32 at 64×64 resolution exhibits slower error growth than DNS at 128×128 and 256×256, particularly in the extrapolation region (shaded, beyond time step 300). The log-scale y-axis highlights SINO's superior long-term stability.}
\label{fig:error_time}
\end{figure}

\end{document}